\documentclass{article}

\usepackage{microtype}
\usepackage{graphicx}
\usepackage{subfigure}
\usepackage{booktabs} % for professional tables
\usepackage{tikz}
\usetikzlibrary{arrows.meta, positioning, fit, calc}
\usepackage{algorithm}
\usepackage{algorithmic}
\usepackage{hyperref}

\usepackage[preprint]{icml2026}

\usepackage{amsmath}
\usepackage{amssymb}
\usepackage{mathtools}
\usepackage{amsthm}
\usepackage{multirow}
\usepackage[capitalize,noabbrev]{cleveref}

\theoremstyle{plain}
\newtheorem{theorem}{Theorem}%[section]
\newtheorem{proposition}[theorem]{Proposition}

\theoremstyle{definition}
\newtheorem{definition}[theorem]{Definition}
\newtheorem{assumption}[theorem]{Assumption}
\theoremstyle{remark}
\newtheorem{remark}[theorem]{Remark}

\usepackage[textsize=tiny]{todonotes}

\icmltitlerunning{Bayesian Meta-Learning with Expert Feedback for Task-Shift Adaptation through Causal Embeddings}

\begin{document}

\twocolumn[
\icmltitle{Bayesian Meta-Learning with Expert Feedback for Task-Shift Adaptation through Causal Embeddings}



\icmlsetsymbol{equal}{*}

\begin{icmlauthorlist}
\icmlauthor{Lotta Mäkinen}{aff1}
\icmlauthor{Jorge Loría}{aff1}
\icmlauthor{Samuel Kaski}{aff1,aff2,aff3}
%\icmlauthor{Firstname4 Lastname4}{sch}
%\icmlauthor{Firstname5 Lastname5}{yyy}
%\icmlauthor{Firstname6 Lastname6}{sch,yyy,comp}
%\icmlauthor{Firstname7 Lastname7}{comp}
%\icmlauthor{Firstname8 Lastname8}{sch}
%\icmlauthor{Firstname8 Lastname8}{yyy,comp}
\end{icmlauthorlist}

\icmlaffiliation{aff1}{Department of Computer Science, Aalto University, Espoo, Finland}
\icmlaffiliation{aff2}{ELLIS Institute Finland, Espoo, Finland}
\icmlaffiliation{aff3}{Department of Computer Science, University of Manchester, United Kingdom}

\icmlcorrespondingauthor{Lotta M\"akinen}{lotta.makinen@aalto.fi}
\icmlcorrespondingauthor{Jorge Lor\'ia}{jorge.loria@aalto.fi}

\icmlkeywords{Bayesian meta-learning, transfer learning, expert knowledge elicitation, structural causal models}

\vskip 0.3in
]

\newcommand{\JL}[1]{\textcolor{purple}{\textsf{[JL: #1]}}}
\newcommand{\LM}[1]{\textcolor{blue}{\textsf{[LM: #1]}}}



\printAffiliationsAndNotice{}  % no special notice (required even if empty)

\begin{abstract}
Meta-learning methods perform well on new within-distribution tasks but often fail when adapting to out-of-distribution target tasks, where transfer from source tasks can induce negative transfer. We propose a causally-aware Bayesian meta-learning method, by conditioning task-specific priors on precomputed latent causal task embeddings, enabling transfer based on mechanistic similarity rather than spurious correlations. Our approach explicitly considers realistic deployment settings where access to target-task data is limited, and adaptation relies on noisy (expert-provided) pairwise judgments of causal similarity between source and target tasks. We provide a theoretical analysis showing that conditioning on causal embeddings controls prior mismatch and mitigates negative transfer under task shift. Empirically, we demonstrate reductions in negative transfer and improved out-of-distribution adaptation in both controlled simulations and a large-scale real-world clinical prediction setting for cross-disease transfer, where causal embeddings align with underlying clinical mechanisms. 
\end{abstract}

\section{Introduction}
\label{sec:introduction}
A fundamental challenge in transfer learning and meta-learning is adapting to new tasks whose data-generating mechanisms differ from those encountered during training. In meta-learning, each task has its own conditional distribution, and \textbf{task-level} distribution shift occurs when the target task is generated by a different underlying mechanism from the source tasks. This shift can severely degrade predictive performance \citep{quinonero2022dataset}, making this issue critical in many domains such as healthcare. Clinically similar manifestations across patients and diseases may arise from distinct biological mechanisms \citep{subbaswamy2020development}: for example, type 1 and 2 diabetes share overlapping symptoms but arise from different mechanisms, namely impaired insulin \textit{production} versus impaired insulin \textit{response}. Hence, models trained across tasks may transfer spurious correlations rather than causal mechanisms, leading to worse performance when deployed on new tasks \citep{wang2019characterizing}. 

On the other hand, under appropriate assumptions, \textbf{causal relationships} remain invariant across environments and are therefore stable under distribution shift \citep{arjovsky2019invariant, pearl2009causality}. If the causal mechanism of each task were known, identifying which source tasks are relevant for a new target task would be straightforward, since tasks with similar mechanisms would be expected to generalize well to one another. In practice, however, the causal mechanisms are \textbf{not} known, and existing causal discovery and inference methods typically require shared feature spaces or joint access to data across tasks \citep{peters2016causal, lorchDiBSDifferentiableBayesian2021a}. The problem is that when source and target tasks originate from different datasets, it is \textit{impossible} to assess the causal similarity between tasks from observational data alone. This situation is prevalent in healthcare settings due to privacy concerns and strict laws.
% This setting is common in healthcare due to privacy and data-sharing constraints. 

Reasoning about causal relationships is a core part of clinical practice. Through differential diagnosis, clinicians routinely compare new patient cases to previously seen ones, reasoning about which underlying mechanisms best explain the observed presentation \citep{eva2005every, pelaccia2011analysis}. Such expert knowledge provides exactly the causal information that is \textit{unidentifiable} from data alone. Despite this, no prior work has leveraged expert judgments about causal similarity to align target and source tasks. 

Although \textbf{meta-learning} is a powerful framework for fast adaptation to new tasks using only a few samples, it still struggles when faced with out-of-distribution tasks \citep{hospedales2021meta}. Most meta-learning approaches optimize a single shared initialization or prior across source tasks, which is then used to adapt to a new target task. As a result, the meta-learner cannot distinguish the source tasks that are most relevant for a given target task, and the shared prior may bias adaptation in a wrong direction when the target task lies outside the source-task distribution, leading to negative transfer. Meta-learning methods that incorporate task similarity, rely on ``closeness'' in the model parameters or learned features, without capturing causal structures \citep{yao2019hierarchically, zhou2021task}. %%% do we know for sure they are unstable? -- JL

Our approach leverages causal task structures to guide transfer under task-level distribution shift. During meta-training, each source task is embedded into a latent causal task embedding space designed to capture stable mechanistic relationships between tasks. These causal task embeddings parameterize the task-specific prior of a Bayesian meta-learner, modulating transfer by the causal relationships between tasks rather than (possibly spurious!) correlations. % both during \textbf{training} and during \textbf{adaptation} to new tasks. 

At deployment, when the embedding of a new target task is unknown, we infer its position in the causal embedding space using pairwise expert task similarity judgments between the target and source tasks. We assume access to a domain expert that provides (noisy) pairwise judgments about the relative causal similarity between tasks, without access to task observations. The inferred target embedding then defines a task-adaptive prior that emphasizes causally aligned source tasks, mitigating negative transfer under distribution shift. We demonstrate the effectiveness of this approach through theoretical analysis and experiments on both synthetic benchmarks and in a real-world clinical cross-disease prediction setting. 

%% In Section~\ref{sec:background} we provide preliminaries, and discuss related work in Section~\ref{sec:related-work}. 
Our main contributions are: \vspace{-0.4cm}
\begin{itemize}
    \item A causally-aware Bayesian meta-learning method that conditions task-specific priors on a causal task embedding space and uses a Bayesian expert preference model over pairwise similarity queries to infer the target-task embedding without requiring access to target-task data, described in Sections ~\ref{sec:meta-learning} and~\ref{sec:expert-inference}.
    \item A theoretical analysis characterizing negative transfer in terms of prior mismatch under task shift, presented in Section~\ref{sec:error-decomposition}.
    \item Numerical experiments on large-scale medical data showing that our approach improves adaptation and reduces negative transfer under distribution shift, compared to alternative methods, shown in Section~\ref{sec:experiments}.
    %, in both synthetic settings and large-scale medical data. %%%% removing the substantially for now
\end{itemize}
% We conclude in Section~\ref{sec:conclusion} by pointing out future directions that stem from our work.
\section{Preliminaries}\label{sec:background}
\subsection{Bayesian Meta-Learning}\label{sec:background-metalearning}
We build on the amortized hierarchical Bayesian meta-learning formulation \citep{ravi_amortized_2018}, which we summarize in this section. % The meta-learning is formalized as a hierarchical Bayesian model over different tasks $t\in\mathcal{T}$. 
The generative model assumes shared global parameters of a neural network and task-specific parameters in the same space. The global parameters are denoted by $\theta \in \mathbb{R}^p$, and the task-specific parameters by $\phi_t \in \mathbb{R}^p$, for $t\in\mathcal{T}$; following the hierarchical structure
\begin{align}
\theta &\sim p(\theta), \label{eq:hyperprior}\\
\phi_t\mid \theta &\overset{i.i.d.}{\sim} \mathcal{N}(\theta,\sigma^2 I_p), \text{ for } t\in \mathcal{T} \text{, and} \label{eq:task-prior}\\
\mathbf{y}_{t,i} \mid \mathbf{x}_{t,i},\phi_t &\overset{indep.}{\sim} p(\mathbf{y}_{t,i} \mid \phi_t,\mathbf{x}_{t,i}),\label{eq:likelihood}
\end{align}
for $i=1,\dots,M_t$.
% Here $\theta, \phi_t$ correspond to neural network parameters in a shared parameter space. 
The task-prior $\phi_t \mid \theta$ ensures that the task-specific parameters remain close to the global parameters $\theta$. In this paper we follow the common choice of independent Gaussian random variables, but this is easily adapted to other general priors. The likelihood for task $t$ is denoted by $p(\mathbf{y}\mid \phi_t,\mathbf{x})$, with parameters encoded by $\phi_t$ (e.g., the mean and variance are functions of $\phi_t$, and $\mathbf{x}$). Each task dataset is given by $\mathcal{D}_t = \{(\mathbf{x}_{t,i},\mathbf{y}_{t,i})\}_{i=1}^{M_t}$ and is 
%$\phi_t \mid \theta \sim \mathcal{N}(\theta, \sigma^2I_p)$, which centers the task-specific parameters around $\theta$. Equivalently, this prior can be written as
%\[
%\phi_t = \theta + \epsilon, \qquad \epsilon \sim \mathcal{N}(0, \sigma^2 I_p).
%\] 
% Each task dataset $\mathcal{D}_t$ is 
further divided into a support set $\mathcal{D}_t^{(s)}$ and a query set $\mathcal{D}_t^{(q)}$, which simulates train-test episodes inside the meta-training. Amortized variational inference is used, following \citet{ravi_amortized_2018}, to learn a global variational distribution $q_{\lambda}(\theta)$ and task-level variational distributions $q_{\psi_t}(\phi_t \mid \mathcal{D}_t^{(s)})$, where the $\lambda$ are the variational global parameters and the $\psi_t$ the variational task-specific parameters of the task posterior. 

Typical approaches follow bi-level optimization objective with nested inner and outer loss functions. The inner task-level loss function is 
\begin{align}\label{eq:inner-loss}
    \mathcal{L}_1(\psi_t, \mathcal{D}_t^{(s)}) 
&:= 
- \mathbb{E}_{q_{\psi_t}(\phi_t \mid \mathcal{D}_t^{(s)})} 
\Big[ \log p(\mathbf{y}_t^{(s)} \mid \mathbf{x}_t^{(s)}, \phi_t) \Big] \nonumber \\
&\quad + D_{\mathrm{KL}}\Big(q_{\psi_t}(\phi_t \mid \mathcal{D}_t^{(s)}) \;\|\; p(\phi_t \mid \theta)\Big),
\end{align}
with outer global loss function given by 
\begin{align}\label{eq:outer-loss}
    \mathcal{L}_2(\psi_t, \mathcal{D}_t^{(q)}, \lambda) 
&:= 
- \mathbb{E}_{q_{\psi_t}(\phi_t \mid \mathcal{D}_t^{(s)})} 
\Big[ \log p(\mathbf{y}_t^{(q)} \mid \mathbf{x}_t^{(q)}, \phi_t) \Big] \nonumber \\
&\quad + D_{\mathrm{KL}}\Big(q_\lambda(\theta) \;\|\; p(\theta)\Big).
\end{align}
The bi-level problem has an outer optimization of
\begin{equation}
    \min_{\lambda} \mathbb{E}_{t \sim p(\mathcal{T})} \left[ \mathcal{L}_2\left(\psi_t^*, \mathcal{D}_t^{(q)}, \lambda\right) \right], 
\end{equation}
where $\psi_t^* = \arg\min_{\psi_t} \mathcal{L}_1(\psi_t, \mathcal{D}_t^{(s)})$ is a solution of the inner optimization.

During meta-training a shared prior $p(\phi_t \mid \theta)$ across tasks is used to learn a posterior distribution of the parameters for each task, and at meta-test time, the model is presented with a new target task $t'$. The task-specific parameters are adapted by minimizing $\mathcal{L}_1(\psi_{t'}, \mathcal{D}_{t'}^{(s)})$ using the shared prior. Predictions are made on held-out test data. A key limitation of this formulation is that when tasks differ in their data-generating mechanisms, enforcing a single global prior for all possible target tasks can lead to worse performance on out-of-distribution target tasks, leading to negative transfer \citep{wang2019characterizing}. In Section \ref{sec:meta-learning} we extend this framework by conditioning the prior $p(\phi_t \mid \theta)$ on causal task embeddings, allowing the prior to vary across tasks. Furthermore, in Section~\ref{sec:error-decomposition} we explore how this affects transfer in traditional methods and how our approach is better.

\subsection{Structural Causal Models}\label{sec:background-causal}
We follow the structural causal model (SCM) notation \citep{pearl2009causality}. A task-specific SCM is a tuple $\mathcal{W}=(\mathbf{V}, \mathbf{U}, F, P_U)$, where $\mathbf{V}=\{X_1,\ldots,X_d\}$ are the endogenous variables, $\mathbf{U}=\{U_1,\ldots,U_d\}$ are exogenous noise variables, $F=\{f_1,\ldots,f_d\}$ is the set of structural equations $X_i = f_i(\mathrm{Pa}(X_i),U_i)$, with $\mathrm{Pa}(X_i) \subseteq \mathbf{V}$ denoting the causal parents of $X_i$, and $P_U$ is the joint distribution over the noise variables. An SCM induces a directed acyclic graph (DAG) where each variable in $\mathbf{V}$ is a vertex and edges encode parent-child relations defined by $\mathrm{Pa}(X_i)$.

\section{Related Work}\label{sec:related-work}
\paragraph{Meta-learning under task-level distribution shift.} 

Meta-learning provides a framework for transferring inductive bias across multiple tasks by learning shared model parameters or priors. Existing approaches include gradient-based methods, such as model agnostic meta-learning \citep{finn2017model}, and Bayesian meta-learning methods, which model tasks through a shared global prior over task-specific parameters and offer a principled way to reason about uncertainty \citep{grant2018recasting, finn2018probabilistic, yoon2018bayesian, ravi_amortized_2018}. However, most meta-learning methods implicitly assume that source and target tasks are identically and independently drawn from a common task distribution. When this assumption is violated, the learned inductive bias becomes misspecified and transfer can even lead to negative transfer \citep{wang2019characterizing}. Recent theoretical works formalize this by showing that meta-generalization error depends explicitly on a divergence between source and target task distributions: implying that performance degrades when target tasks move farther from source tasks \citep{fallah2021generalization, jose2021information, chen2021generalization}. In order to mitigate heterogeneity among tasks, several approaches cluster source tasks or learn multiple task-specific priors or initializations \citep{zhou2021task, yao2019hierarchically}. Other works embed tasks into latent spaces derived from data and compute task similarities to guide transfer \citep{achille2019task2vec, chen2020mate}. However, these approaches define the task similarity using observed features or model parameters, which are unstable under distribution shift. More recent work has explored modeling similarity between source tasks at the level of causal mechanisms to group tasks during training to improve generalization within the same task distribution \citep{wharrie_bayesian_2024}. Crucially, none of these methods address how to align target and source tasks and infer the causal relationships between them, particularly when no target task data is available. 

\paragraph{Causal inference and task heterogeneity.}
From a causal perspective, task heterogeneity arises from differences in the underlying data-generating processes \citep{pearl2009causality}. If the task-specific structural causal models (SCMs) are fully known, tasks could be compared directly by differences in their directed acyclic graphs (DAGs) or structural equations. Recent Bayesian causal discovery methods, such as DiBS \citep{lorchDiBSDifferentiableBayesian2021a}, aim to infer full posterior distributions over causal graph structures from data. While powerful, inferring complete SCMs is often impractical in real-world settings, particularly in high-dimensional problems or in the presence of latent confounding. Instead of full causal discovery, many causal inference (CI) approaches estimate task-specific causal effects. Methods such as instrumental variable estimation \citep{angrist1996identification} and invariant causal prediction \citep{peters2016causal} produce estimates of the direction and magnitude of causal relationships, which can be summarized as causal representation vectors \citep{scholkopf2021toward}. These representations can serve as proxies for the underlying causal mechanisms, enabling task comparison even when full SCM discovery is impossible. However, existing CI methods require access to task-level observational data, and cannot be applied to assess similarity between disjoint source and target tasks using data alone. 
\paragraph{Expert knowledge in learning systems.}
Incorporating expert knowledge into learning systems has been studied in a variety of settings, particularly through preference learning and pairwise comparison methods. Ranking-based models are a classical way of eliciting relative information from experts, formalized by \citet{thurstone_law_1927} and \citet{bradley_rank_1952}. Preference learning extends these models by placing a probabilistic model over latent utilities or similarities and modeling expert feedback as noisy but consistent relative judgments \citep{chu2005preference}. These models have been naturally connected to active learning for selecting the most informative queries \citep{houlsby2011bayesian}. Expert knowledge has also been incorporated into CI, particularly in the context of learning causal structures. Prior work has explored learning causal Bayesian networks with expert elicitation \citep{constantinou2016integrating}. Recently, \citet{bjorkman2025incorporatingexpertknowledgebayesian} proposed a Bayesian experimental design approach to elicit expert knowledge for discovering mixtures of DAGs in heterogeneous domains. However, these approaches identify causal structure \textit{within} a dataset, rather than aligning a new task to previously learned tasks when target data is unavailable. In meta-learning, incorporating expert knowledge remains an open challenge highlighted by recent reviews \citep{vettoruzzo2024advances}. In particular, existing meta-learning methods do not leverage expert feedback to align target tasks with source tasks at adaptation time. In contrast, our approach uses expert pairwise similarity judgments to align a (previously unseen) task with source tasks based on the causal mechanisms, enabling transfer under task-level shifts. % in a meta-learning context.

\section{Incorporating Causal Embeddings to a Bayesian Meta-Learning Model}\label{sec:meta-learning}
\subsection{Causal Embeddings}
Assume access to latent \emph{causal embeddings} $z\in \mathcal{Z}\subseteq \mathbb{R}^d$, equipped with a metric $d_{\mathcal{Z}}: \mathcal{Z} \times \mathcal{Z} \rightarrow \mathbb{R}_{\geq 0}$. We use the Euclidean metric for simplicity, although our framework is agnostic to the choice of metric.
\begin{definition}% [Causal task embedding]
For each task $t \in \mathcal{T}$, we define a causal task embedding $z_t \in \mathcal{Z}$, a low-dimensional vector representation summarizing the causal mechanistic structure of the task.
\end{definition}
The embedding $z_t$ is constructed by mapping the vector representation of the underlying SCM. To construct said vector representation, we use invariant causal prediction \citep[ICP;][]{peters2016causal} and an instrumental variables approach with Mendelian randomization \citep[MR;][]{didelez2007mendelian}, both described in App.~\ref{sec:causal_methods}. % We also explore inferring causal embeddings from expert feedback, and the expert inference model is described in Section \ref{sec:expert-inference}. %%% I don't think we need to say this yet.

Task embeddings are constructed to summarize causal mechanisms. As such, distances in $\mathcal{Z}$ can be interpreted as causal relatedness. We introduce the following notion of proximity in the embedding space.
\begin{definition}%[$\varepsilon$-similar tasks]
\label{def:eps-similarity}
    Two tasks $t$ and $t'$ are $\varepsilon$-similar if 
    %\[
    $d_{\mathcal{Z}}(z_t, z_{t'}) \le \varepsilon,$
    %\]
    where $d_{\mathcal{Z}}$ is the metric on $\mathcal{Z}$.
\end{definition}
For a specific task $t\in \mathcal{T}$, the true outcome generating distribution is denoted by $p_t(y\mid \mathbf{x})$. The following key assumption helps link causal structure to predictive behavior.
\begin{assumption}% [Causal--predictive alignment]
\label{as:smoothness}
 Given two $\varepsilon$-similar tasks $t,t'$, their predictive conditional distributions $p_{t}(y \mid \mathbf{x})$ and $p_{t'}(y \mid \mathbf{x})$ are at a distance of $\delta_\varepsilon$. 
\end{assumption}
Assumption~\ref{as:smoothness} is a smoothness condition on $p_t(y\mid \mathbf{x})$ with respect to the latent embeddings $z_t$. Formalizing the idea that similarity in the underlying causal mechanisms can be translated into similarity in the predictive behavior. 

In the rest of the article, we use the task embeddings $z_t$ to condition the task-specific prior, such that the geometry of $\mathcal{Z}$ implicitly modulates information transfer across tasks. 
\subsection{Causally-Aware Bayesian Meta-Learning}
% \JL{(to self): removed this paragraph, need to ensure we are not missing anything. 
%  Our focus is on binary classification tasks, where the goal is to predict $y\in \{0,1\}$ given an input~$\mathbf{x}$.
% }
We extend the Bayesian meta-learning framework from Section~\ref{sec:background-metalearning} by conditioning the task-specific prior on causal embeddings. Eq.~\eqref{eq:task-prior}, the standard task-specific prior, centers all tasks around the same global parameter $\theta$. We modify it to the embedding-aware prior
\begin{equation}\label{eq:task-prior-new}
    \phi_t \mid z_t, \theta \overset{indep.}{\sim} \mathcal{N}(\theta + Wz_t, \sigma^2 I_p), \text{ for } t\in\mathcal{T},
\end{equation}
where $W \in \mathbb{R}^{p \times d}$ is a learnable weight matrix mapping embeddings to parameter space, and $\sigma^2$ is a prior variance. The prior mean $\theta + Wz_t$ is task-specific through the embedding $z_t$, while the prior variance is shared across all tasks. 

This linear parameterization admits a geometric interpretation (Figure \ref{fig:spaces-illustration}). The global parameter $\theta$ is a starting point in the parameter space $\Phi$, and each task embedding $z_t$ induces a displacement $Wz_t$ within the affine subspace $\{\theta + Wz: z\in \mathcal{Z}\}$, arriving at a task-specific prior mean at $\theta + Wz_t$. Tasks that lie close in the embedding space $\mathcal{Z}$ induce similar prior means; this aligns with Assumption \ref{as:smoothness}. When $z_t=0$ or $W=0$, the prior reduces to $\mathcal{N}(\theta,\sigma^2I)$, the standard meta-learning prior. 
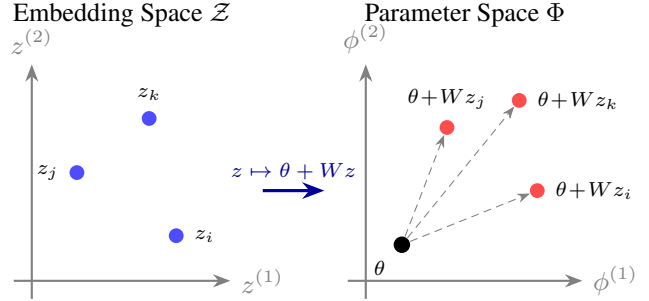
\begin{figure}[!t]
\centering
\begin{tikzpicture}[
    >=Stealth,
    scale=1.2,
    every node/.style={scale=1.1},
    embedding point/.style={circle, fill=blue!70, inner sep=1.8pt},
    prior mean/.style={circle, fill=red!70, inner sep=1.8pt},
    global mean/.style={circle, fill=black, inner sep=2pt},
    axis/.style={->, thick, gray},
    map arrow/.style={->, thick, shorten >=2pt, shorten <=2pt},
    displacement/.style={->, densely dashed, gray},
    font=\small
]
%embedding space, left panel
\begin{scope}[local bounding box=left panel]
    %axes
    \draw[axis] (-0.2,0) -- (2.2,0) node[right] {$z^{(1)}$};
    \draw[axis] (0,-0.2) -- (0,2.4) node[above] {$z^{(2)}$};
    
    %task embeddings
    \node[embedding point, label={[font=\scriptsize]right:$z_i$}] (z1) at (1.6,0.5) {};
    \node[embedding point, label={[font=\scriptsize]left:$z_j$}] (z2) at (0.5,1.2) {};
    \node[embedding point, label={[font=\scriptsize]above:$z_k$}] (z3) at (1.3,1.8) {};
    
    %panel header
    \node[above, font=\small] at (1.0,2.7) {Embedding Space $\mathcal{Z}$};
\end{scope}
%arrow between panels
\draw[map arrow, very thick, blue!60!black] (2.5,1.0) -- (3.3,1.0) 
    node[midway, above, font=\scriptsize] {$z \mapsto \theta + Wz$};
%parameter space, right panel
\begin{scope}[xshift=3.7cm, local bounding box=right panel]
    %axes
    \draw[axis] (-0.2,0) -- (2.4,0) node[right] {$\phi^{(1)}$};
    \draw[axis] (0,-0.2) -- (0,2.4) node[above] {$\phi^{(2)}$};
    
    %theta
    \node[global mean, label={[font=\scriptsize]below left:$\theta$}] (theta) at (0.4,0.4) {};
    
    %phi_t
    \node[prior mean, label={[font=\scriptsize]right:$\theta\!+\!Wz_i$}] (phi1) at (1.9,1.0) {};
    \node[prior mean, label={[font=\scriptsize]above:$\theta\!+\!Wz_j$}] (phi2) at (0.9,1.7) {};
    \node[prior mean, label={[font=\scriptsize]right:$\theta\!+\!Wz_k$}] (phi3) at (1.7,2.0) {};
    
    %arrows from theta to phi_t
    \draw[displacement] (theta) -- (phi1);
    \draw[displacement] (theta) -- (phi2);
    \draw[displacement] (theta) -- (phi3);
    
    %panel header
    \node[above, font=\small] at (1.1,2.7) {Parameter Space $\Phi$};
\end{scope}
\end{tikzpicture}
\caption{Illustration of the causal embedding space (\emph{left}) and the mapping to the parameter space (\emph{right}). Each task $t$ has an embedding $z_t$, encoding its causal structure. The linear map $z \mapsto \theta + Wz$ transforms embeddings into task-specific priors, enabling tasks with similar causal mechanisms to have similar priors.}
\label{fig:spaces-illustration}
\end{figure}

Conditioning the prior mean allows task-specific inductive biases to vary smoothly across tasks as a function of their causal similarity, influencing the direction in which the model adapts during task-specific learning. In contrast, keeping the prior variance global isolates the effect of task similarity on the prior mean, promotes stable adaptation across tasks, and simplifies the theoretical analysis of generalization under task shift. Although more expressive mappings could be considered (e.g., neural networks), linear mappings yield favorable theoretical properties and enable analysis of the transfer between task similarity while remaining sufficient to capture task-level structures in our setting. 

We instantiate the proposed model within the standard amortized Bayesian meta-learning framework, described in Section~\ref{sec:background-metalearning}. Consider a dataset $\mathcal{D}_t = \{(\mathbf{x}_{t,i},{y}_{t,i})\}_{i=1}^{M_t}$ with $M_t$ samples for task $t$ drawn from this distribution, and assume that the tasks are separated into source and target, denoted by $\mathcal{T} = \mathcal{T}_{\text{source}} \cup \mathcal{T}_{\text{target}}$.
During meta-training, the model observes a set of source tasks $\mathcal{T}_{\text{source}}$ sampled from a task distribution $t\sim p(\mathcal{T_{\text{source}}})$. For each source task $t$, we assume a causal task embedding $z_t$ is available, and use it to construct an embedding-conditioned prior $p(\phi_t \mid z_t, \theta)$. Task-specific parameters are adapted on support data $\mathcal{D}_t^{(s)}$ by minimizing the inner variational objective 
\begin{align*}\label{eq:inner-loss2}
    \mathcal{L}_1(\psi_t, \mathcal{D}_t^{(s)}, z_t) 
    &:= - \mathbb{E}_{q_{\psi_t}(\phi_t \mid \mathcal{D}_t^{(s)})} 
    \Big[ \log p({y}_t^{(s)} \mid \mathbf{x}_t^{(s)}, \phi_t) \Big] \nonumber \\
    &\quad + D_{\mathrm{KL}}\Big(q_{\psi_t}(\phi_t \mid \mathcal{D}_t^{(s)}) \;\|\; p(\phi_t \mid z_t, \theta)\Big),
\end{align*}
instead of Eq.~\eqref{eq:inner-loss}.
Conditioning the prior on $z_t$ directly influences the KL regularization term in the inner objective ensuring the posterior does not drift too far from a causally aligned prior. The task predictor is parameterized by a Bayesian neural network, where the task parameters $\phi_t$ correspond to the weights of a task prediction head. Full algorithms and optimization details are provided in App.~\ref{app:algorithms}. %% correcting reference -- JL

At meta-test time, the model is presented with a previously unseen target task $t'$ sampled from its task distribution $t' \sim q(\mathcal{T}_{\text{target}})$. If $p(\mathcal{T_{\text{source}}}) \neq q(\mathcal{T}_{\text{target}})$, task-level distribution-shift occurs and the target task distribution will not be aligned with the source task distribution. The target task embedding $z_{t'}$ is used to construct a task-specific prior $p(\phi_{t'} \mid z_{t'}, \theta)$, which guides adaptation to the target task using its observed support data. For adaptation and prediction in new tasks we use Alg.~\ref{alg:2level_adapt_entmax_clean}. %% making explicit reference to the algorithm so it is not new when  including the full algorithm with adaptation.
% Model performance is evaluated and predictions are made on held-out query data from the same target task.  %%% this reads like it would be in the experiment section.

\section{Expert-Guided Inference of Target Task Embeddings}\label{sec:expert-inference}
In a realistic deployment scenario, the source and target datasets are not simultaneously available. Meaning that the source and target tasks ($t$ and $t'$, respectively) cannot be easily embedded into $\mathcal{Z}$ at training time. We assume that it is possible to embed the source tasks into $\mathcal{Z}$, and that the target task embedding $z_{t'}$ is unknown at deployment time, which is key to make predictions in our model.
% If $t \in \mathcal{T}_{\text{source}}$ and $t' \in \mathcal{T}_{\text{target}}$ originate from the same dataset $\mathcal{D}$, then all tasks can be embedded in the common causal embedding space using causal discovery methods. However, in realistic deployment settings, the model is often applied to a separate target dataset $\mathcal{D}_{\text{target}}$, distinct from the source dataset $\mathcal{D}_{\text{source}}$. Simultaneous access to source and target data is not always possible, due to privacy regulations, lack of permits of the deploying team, or simply because the datasets are kept in different secure environments. 

To address this challenge, we propose inferring the target task embedding using structured feedback from a \emph{domain expert}. The key assumption is that while the expert does not have access to the underlying data, they have domain knowledge of the underlying causal mechanisms. As such, they can indicate which of two tasks is closer to the new task. %%% modified this phrase to match what we are doing-- JL
For example, a doctor can usually assess the relative similarity between two diseases or two patients based on their medical expertise, since that is the reasoning typically used in their daily work. We elicit expert knowledge through pairwise similarity comparisons between source tasks and the target task. Then we use these comparisons to infer the target task embedding $z_{t'}$ by aligning it to the embeddings of the source tasks. 

We formalize expert feedback using comparison queries in pairs. Queries are denoted by $\xi_b = (i_b, j_b)$, for ${b\in \{1,\ldots,B\}}$, where $i_b, j_b \in \mathcal{T}_{\text{source}}$ are two source tasks with known embeddings. At each iteration $b$, the expert is asked:
\begin{quote}
    Is source task $i_b$ more similar to target task $t'$ than source task $j_b$?
\end{quote}
The expert response is denoted by $c_b \in \{0,1\}$, where ${c_n = 1}$ indicates that task $i_b$ is judged to be closer to the target task than $j_b$. Collecting $B$ such queries yields a dataset of expert feedback $\mathcal{C} = \{(\xi_b, c_b)\}_{b=1}^B$.

Let $z_{t'} \in \mathbb{R}^d$ denote the latent target task embedding. % and let $d_{\mathcal{Z}}(\cdot,\cdot)$ denote the Euclidean distance between the embeddings. 
For a query $\xi= (i, j)$, we define the relative dissimilarity between the source tasks with respect to the target task as
\begin{align*}
\Delta(\xi;z_{t'}) 
\;&=\; d_{\mathcal{Z}}(z_{t'}, z_{j})-d_{\mathcal{Z}}(z_{t'}, z_{i}) \\
&= \|z_{t'} - z_{j}\|_2 \;-\; \|z_{t'} - z_{i}\|_2.
\end{align*}
The quantity $\Delta(\xi;z_{t'})$ is a deterministic function of the known embeddings ($z_{i}$ and $z_{j}$) and encodes their relative geometry in the task-embedding space. A positive value of $\Delta(\xi;z_{t'})$ indicates that source task $i$ is closer to the target task $t'$ than $j$ is to $t'$.

Although the value of $\Delta(\xi;z_{t'})$ is deterministic (given the $z$s) the expert could \textit{potentially} give incorrect responses and as such we assign a likelihood to the responses. To this end, the expert responses are modeled by a probit likelihood; a standard approach in preference learning \citep{chu2005preference}. The likelihood for each comparison is 
\begin{equation}\label{eq:expert-likelihood}
p(c = 1 \mid z_{t'}, \xi, \tau)
\;=\;
\Phi\!\left( \tau\;\Delta(\xi;z_{t'}) \right),
\end{equation}
where $\Phi(\cdot)$ denotes the standard normal cumulative distribution function and $\tau > 0$ controls the noise level in expert judgments. We fix $\tau=1$ during inference (to avoid identifiability issues) and assign a Gaussian prior on the target-task embedding: $z_{t'} \sim \mathcal{N}(0, I_d)$. % When simulating expert feedback, we generate comparisons with different noise levels $\tau$ and conduct a sensitivity analysis to evaluate the effect of expert noise.  

Given a set of expert responses $\mathcal{C}$, the posterior is
\begin{equation*}
p(z_{t'} \mid \mathcal{C})
\;\propto\;
p(z_{t'}) \prod_{n=1}^N p(c_n \mid z_{t'}, \xi_n, \tau),
\end{equation*}
which we approximate using stochastic variational inference with $q_\varphi(z_{t'}) = \mathcal{N}\left(\mu_q, \sigma_q^2 I_d\right)$, as the approximating family.

To select the queries that will be presented to the expert we use Bayesian active learning by disagreement \citep[BALD;][]{houlsby2011bayesian}. At each iteration, we select the query $\xi=(i,j)$ that maximizes the expected information gain (EIG) about the target embedding $z_{t'}$:
\begin{equation}
\xi^\star
=
\arg\max_{\xi}
\;\mathrm{EIG}(\xi), 
\;
% \qquad
\mathrm{EIG}(\xi)
=
I(y_\xi; z_{t'} \mid \mathcal{C}), 
\end{equation}
where $I(\cdot;\cdot)$ denotes the mutual information and $y_\xi$ is the (future) expert response to query $\xi$. We use the objective in BALD \cite{houlsby2011bayesian}, with the decomposition
\begin{align*}
\mathrm{EIG}(\xi)&=H[y_\xi]
- \mathbb{E}_{z_{t'} \sim p(z_{t'}\mid \mathcal{C})}\!\left[ H[y_\xi \mid z_{t'}] \right], % \\
%&\approx H[y_\xi] - \mathbb{E}_{z_{t'} \sim q_\varphi(z_{t'})}\!\left[ H[y_\xi \mid z_{t'}] \right], 
%%% removed   \mid \mathcal{C}] for consistency with notation.
\end{align*}
where $H[p]=-p\log p-(1-p)\log (1-p)$ is the binary entropy function. Both entropy terms are approximated by sampling from the current variational posterior $q_\varphi(z_{t'})$. The complete active learning algorithm is in App.~\ref{app:expert-feedback-model}.

% \subsection{Incorporating Expert Knowledge into Meta-Learning with New Tasks}\label{sec:}
Finally, in Algorithm~\ref{alg:full_meta} we present the full meta-learning approach that incorporates expert knowledge for prediction in new tasks via causal embeddings. This algorithm combines the meta-learning approach we developed in Section~\ref{sec:meta-learning}, with the expert knowledge and queries developed in this section, and adds the adaptation step. An implementation is freely available in \href{github.com/lottamakinen/causal-meta-learning}{github.com/lottamakinen/causal-meta-learning}. %%%% Remember to add a zip file to the submission!
\begin{algorithm}[!h]
\caption{Causally-aware meta-learning with expert in the loop to predict in a new task}
\label{alg:full_meta}
\begin{algorithmic}[1]
\REQUIRE source datasets $\mathcal{D} = \{\mathcal{D}_t\}_{t\in \mathcal{T}_{\mathrm{source}}}$, target dataset $D_{t'}$, budget of expert queries $B$, expert comparisons $\mathcal{C}$
\STATE \textbf{Obtain} approximate posterior $q_{\phi}(z\mid \mathcal{D})$ using Alg.~\ref{alg:2level_meta_clean} %%% \{\phi_t\},\theta\mid \mathcal{D}
    \STATE \textbf{Actively} learn $\hat{z}_{t'}=\mathbb{E}_q[z_{t'}\mid \mathcal{C}]$ with Alg.~\ref{alg:expert-feedback} 
\STATE \textbf{Adapt} posterior to the target task $t'$, with Alg.~\ref{alg:2level_adapt_entmax_clean}.
\end{algorithmic}
\end{algorithm}
\section{Error Decomposition and Negative Transfer Under Task-Level Shift} \label{sec:error-decomposition}
We analyze the adaptation to a target task $t'\sim q(\mathcal{T}_{\text{target}})$ under a task-level shift ($p(\mathcal{T}_{\text{source}}) \neq q(\mathcal{T}_{\text{target}})$). Our analysis focuses on (i) bounding the prior induced risk, which reflects how well the transferred inductive bias aligns with the target task, and (ii) the mitigation of negative learning it induces.
%Our analysis separates two sources of target-task risk: (i) a prior-induced risk, that reflects how well the transferred inductive bias from source tasks aligns with the target task, 
%and (ii) an adaptation-dependent component capturing finite-sample effects, optimization, and model misspecification.  %%%%%% Why mention this second part if we are not doing it?
% We focus on bounding the prior-induced risk and how it mitigates negative transfer. 
% \paragraph{Prior-induced risk.}

Denote by $\mathcal{R}_{t'}(\phi)=\mathbb{E}_{(x,y)\sim p_{t'}} [\ell(\phi;x,y)]$ the population risk on task $t'$. For the prior (Eq.~\eqref{eq:task-prior-new}), the prior-induced risk is 
\begin{equation*}
\bar{\mathcal R}_t(z)
= \mathbb{E}_{\phi \sim p(\phi\mid z)}[\mathcal R_t(\phi)].
\end{equation*}
Define the mismatch relative to the source-task-induced prior with embedding $\bar{z}:=|\mathcal{T}_\mathrm{source}|^{-1}\sum_{t\in\mathcal{T}_\mathrm{source}}z_t$, as
\[
\mathcal{E_{\mathrm{prior}}}(t') := \big|\bar{\mathcal{R}}_{t'}(\hat{z}_{t'})- \bar{\mathcal{R}}_{t'}(\bar{z})\big|.
\]
Next, we formalize the intuition that nearby tasks in causal embedding space induce similar risks, and provides a decomposition of the generalization error in its three main components: the expert error, the causal discovery error, and the out-of-distribution error.
 % priors over predictors. 
\begin{proposition}% [$\varepsilon$-continuity of prior-induced risk]
\label{prop:eps-continuity}
Assume the loss is bounded, $\ell(\phi;x,y)\in[0,M]$. If $z_1$ and $z_2$ are $\varepsilon$-similar, then for any task $t$,
\begin{equation*}
\big|\bar{\mathcal R}_t(z_1)-\bar{\mathcal R}_t(z_2)\big| \le \frac{M\|W\|}{2\sigma}\,\varepsilon.
\end{equation*}
% where $\sigma^2$ is the prior variance.  %%% it should be clear from the original eqs. what sigma is.
% \begin{proposition}% [Prior mismatch under embedding error]
% \label{prop:prior-mismatch}
For task $t'$, the mismatch satisfies
\begin{equation}\label{eq:main-bound}
\mathcal E_{\mathrm{prior}}(t')
\;\le\;
\frac{M\|W\|}{2\sigma}
\big(
\varepsilon_{\mathrm{expert}}
+
\varepsilon_{\mathrm{causal}}
+
\varepsilon_{\mathrm{OOD}}
\big),
\end{equation}
where $\|W\|$ is the spectral norm of the embedding weight matrix, ${\varepsilon_{\mathrm{OOD}} = \|z_{t'}-\bar{z}\|}, {\varepsilon_{\mathrm{causal}} = \|\tilde z_{t'} - z_{t'}\|}$, and ${\varepsilon_{\mathrm{expert}} = \|\hat z_{t'} - \tilde z_{t'}\|};$ using $\tilde{z}_{t'}$ as the embedding recovered via causal discovery, and $\hat{z}_{t'}$ as the embedding used at deployment and inferred from expert feedback.
\end{proposition}
The proof is given in App.~\ref{app:eps-continuity-full}. % The second part of this result bounds the generalization error induced by the prior mismatch.
%\end{align*}
Eq.~\eqref{eq:main-bound} shows that prior mismatch grows at most \textit{linearly} with geometric displacement in the embedding space. In particular, if the target task is $\varepsilon$-similar to the source tasks, then the induced prior mismatch is $O(\varepsilon)$.  Furthermore, prior mismatch improves when (i) the target task is closer to the source-induced global prior (small $\varepsilon_{\mathrm{OOD}}$), (ii) causal discovery is accurate (small $\varepsilon_{\mathrm{causal}}$), and (iii) expert inference is accurate (small $\varepsilon_{\mathrm{expert}}$). As such it provides three different ways to control the generalization error, instead of a single one.
% The full derivation is given in Appendix~\ref{app:error-decomposition}.

Negative transfer occurs when an inductive bias from the source tasks induces a worse performance (higher risk) compared to training \textit{without} transfer. For a method $X$ and task $t'$, negative transfer is
\begin{equation}
\mathrm{NT}_X(t') := \mathcal R_{t'}(\hat\phi_{t'}^{X}) - \mathcal R_{t'}(\hat\phi_{t'}^{\mathrm{NT}}), \label{eq:neg-transf}
\end{equation}
where $\hat\phi_{t'}^{\mathrm{NT}}$ is a predictor trained only on the target task. 

Finally, we present Theorem~\ref{th:negative-transfer}, providing an interpretable condition to establish that negative transfer is \textbf{mitigated} by using a causal prior and a good-enough expert.
\begin{theorem}% [Mitigation of negative transfer]
\label{th:negative-transfer}
Under mild assumptions, stable adaptation, and a well-conditioned embedding map $W$, if
\begin{equation*}
\varepsilon_{\mathrm{expert}} + \varepsilon_{\mathrm{causal}}
\;\leq\;
C\cdot\varepsilon_{\mathrm{OOD}},
\end{equation*}
for a constant $C$. Then, the causal prior mitigates negative transfer relative to the global prior, i.e.,
\begin{equation*}
\mathrm{NT}_{\mathrm{causal}}(t') \leq \mathrm{NT}_{\mathrm{glob}}(t').
\end{equation*}
\end{theorem}
The proof and full statement are in App.~\ref{app:negative-transfer}. The conditions of Theorem~\ref{th:negative-transfer} require that the error from locating the target task in the causal embedding space is smaller than the level of task-level shift between the source and target tasks.
\section{Evaluation under Task-level Shifts}\label{sec:experiments}
\subsection{Synthetic Setting}\label{sec:synthetic-expertiments}
We analyze the negative transfer, under distribution shift, of our method synthetic datasets. The data generation mechanisms are detailed in App.~\ref{app:data-generation}.

\paragraph{Experiment 1: Preventing Negative Transfer under Task Shift.} We evaluate whether conditioning the prior on causal task embeddings prevents negative transfer under task-level distribution shift in the presence of spurious correlations. For this, we compare three models: a global prior without task conditioning (meaning that it does not observe the $z_t$s), correlation-based task embeddings ($z_t$s inferred via correlations), and causal embeddings. Here we focus on the effect of OOD task-level shift. For this, we assume perfect causal discovery $\varepsilon_{\text{causal}}=0$, and embeddings obtained with a perfect expert ($\varepsilon_{\text{expert}}=0$). That is, the true $z_{t'}$ is known. In App.~\ref{app:additional-causal-noise} we show that the method stays robust under imperfect causal discovery when $\varepsilon_{\text{causal}}>0$. 
\begin{figure}[t]
    \centering
    \includegraphics[width=0.9\linewidth]{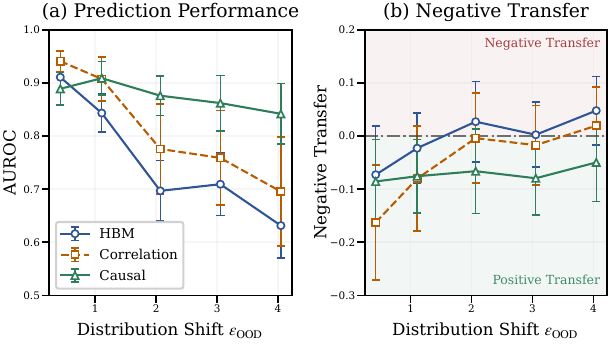}
    \caption{Performance of meta-learning models under increasing task-shift for Experiment 1. a) AUROC as a function of the distribution shift $\varepsilon_\mathrm{OOD}$. b) Change in log loss relative to no transfer baseline (BNN), where negative values indicate positive transfer and positive values negative transfer. Error bars denote standard deviation across 30 runs.   % \LM{Check font size, axis names, remove white padding to make bigger, legend names} \JL{Ensure that all the error bars fit inside the right hand side. Y-axis of right panel can just say 'no-transfer'. Unify the label for our method (e.g., call it 'causal'), same for the 'global prior'.}
    }
    \label{fig:experiment-1}
\end{figure}

Figure~\ref{fig:experiment-1}a shows the AUROC for different levels of task shift ($\varepsilon_\mathrm{OOD}$). The global prior and the correlation-based embeddings exhibit a degradation in predictive performance as the target task moves farther from the source distribution. In contrast, using the causal embeddings yields a more stable AUROC across all levels of task shift. Figure~\ref{fig:experiment-1}b illustrates this behavior through negative transfer (Eq.~\eqref{eq:neg-transf}), measured as the change in log-loss relative to a non-transfer baseline given by independently-trained task-specific Bayesian neural networks. Positive values indicate worse performance than the no-transfer baseline. The global prior has increasing negative transfer as $\varepsilon_\mathrm{OOD}$ grows, reflecting a mismatch under larger task shifts. Correlation-based embeddings also exhibit larger negative transfer for larger shifts: highlighting the instability of spurious correlations. This instability is further shown by the increasing variance across runs, as indicated by the widening error bars at higher shift levels. In contrast, the causal prior induces positive transfer across all OOD levels. The causal prior effectively mitigates negative transfer even under substantial task-level shift and spurious correlations, thereby explaining the stable performance. 

Notably, correlation-based embeddings outperform the other approaches in the in-distribution case (i.e.,~${\varepsilon_{\mathrm{OOD}}=0.5}$), where spurious correlations are predictive and therefore beneficial. However, this advantage disappears under task shift when these correlations no longer hold, while causal embeddings remain robust. 

%\LM{Error bars, volatility of correlation embeddings. In distribution tasks spurious correlations are predictive -> correlations better than causal, break under OOD shift}

%Figure~\ref{fig:experiment-1}a) shows the negative transfer  as a function of task-level shift $\varepsilon_{\text{OOD}}$. Positive values indicate a worse performance than a non-transfer baseline. For the non-transfer baseline, we use independently trained task-specific Bayesian neural networks. See Appendix~\LM{REF} for details.

%Consistent with the prior-mismatch term in Eq.~\eqref{eq:main-bound}, the global prior induces increasing negative transfer as the target task moves farther from the source task distribution. Correlation-based embeddings also exhibit increasing negative transfer, reflecting that spurious correlations do not remain stable under the task shift. In contrast, conditioning the prior on causal task embeddings keeps $\mathrm{NT}_X(t')$ close to zero across all OOD levels, indicating that causal conditioning prevents negative transfer even under substantial task shift and spurious correlations. 

%Figure~\ref{fig:auroc_vs_ood} shows how this effect translates into predictive performance. Both the global prior and correlation-based embeddings degrade as the shift increases, while the AUROC of the causal embeddings remains stable across OOD levels. An additional synthetic experiment without spurious correlations, where feature correlations align with causal mechanisms, is reported in  Appendix~\ref{app:additional_synthetic}.
\paragraph{Experiment 2: Using Expert Inferred Target Task Embeddings.}
We next consider a more realistic deployment scenario in which the causal embedding of the target task  is not known and must be inferred using expert feedback, as described in Section~\ref{sec:expert-inference}. We simulate an expert that provides noisy pairwise similarity judgments based on the true underlying causal task embeddings $z_t$. Expert responses are generated according to the likelihood in Eq.~\eqref{eq:expert-likelihood}, with expert reliability controlled by the parameter $\tau_\mathrm{expert}$. We set $\tau_\mathrm{expert}=2$, and report sensitivity analyses to expert noise in App.~\ref{app:additional_expert_noise}. To reflect realistic imperfection in causal discovery, we further corrupt the source-task embeddings used by the meta-learner with additive Gaussian noise $\eta \sim \mathcal{N}(0,(0.5)^2)$. The pairwise queries presented to the expert are selected using BALD. App.~\ref{app:additional_bald_random} shows comparisons to random query selection.

Figure~\ref{fig:experiment-2}a compares expert-inferred causal embeddings against meta-learning baselines under increasing task shift. The baselines included are a model agnostic meta-learning \citep[MAML;][]{finn2017model}, a deep kernel transfer model \citep[DKT;][]{patacchiola2020bayesian}, and an amortized hierarchical Bayesian meta-learning with a global prior \citep[HBM;][]{ravi_amortized_2018}. Both MAML and the global prior exhibit degraded performance as $\varepsilon_\mathrm{OOD}$ increases. In contrast, conditioning the prior on expert-inferred causal embeddings yields performance close to the causal oracle embeddings without expert inference across all OOD levels. This indicates that expert feedback enables accurate recovery of a target task embedding that is well-aligned with the underlying causal mechanism, even when causal discovery is imperfect. DKT performs strongly across all shift levels in this synthetic setting. As a metric-based method utilizing Gaussian processes, DKT is well-suited to low-dimensional problems, and therefore serves as a strong baseline in this experiment. As we show in Section~\ref{sec:ukbb}, this advantage diminishes in more complex real-world settings, where optimization-based and Bayesian meta-learning methods are more effective. Figure~\ref{fig:experiment-2}b shows the effect of expert query budget $B$ on performance for a single task ($\varepsilon_{\text{OOD}}=4.0$). Increasing the number of expert queries consistently improves AUROC, reflecting progressively more accurate inference of the target task embedding. 
\begin{figure}[t]
    \centering
    \includegraphics[width=0.9\linewidth]{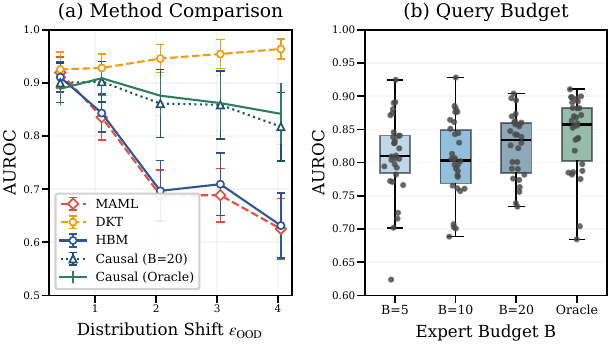}
    \caption{Method comparison with expert inferred embeddings for Experiment 2 across 30 runs. a) Average AUROC across increasing distribution shift $\varepsilon_\mathrm{OOD}$ levels for our method against baselines, error bars denote SD. b) Effect of expert query budget on prediction performance, for a single task (${\varepsilon=4.0}$).  % \LM{Change causal oracle line style match Figure 3. Name global prior as HBM similar to UKBB exp.}
    }
    \label{fig:experiment-2}
\end{figure}

\subsection{Cross-Disease Meta-Learning in the UK Biobank}\label{sec:ukbb}
Next, we evaluate our proposed method on clinical prediction tasks derived from the UK Biobank (UKBB), a large-scale population cohort with longitudinal electronic health records. In contrast to the settings in Section~\ref{sec:synthetic-expertiments}, disease prediction tasks exhibit heterogeneous populations, complex pathways, and unknown confounders.

We define each task as a binary disease prediction problem, where the goal is to predict the presence of a certain disease after an index year using patient-level covariates, clinical history, and treatment endpoints. Tasks correspond to different disease endpoints, in the cardiovascular and respiratory systems.
Details of the dataset are in App.~\ref{app:ukbb}, along with a description of the diseases. This setting ensures distribution shift between the tasks, since diseases differ in prevalence, causal mechanisms, and risk factors. 

Since the causal mechanisms are not observed, we construct embeddings using causal inference (CI) methods in the source tasks, using longitudinal data. We use three approaches: an instrumental variable approach with Mendelian Randomization \citep[MR;][]{didelez2007mendelian};  an invariant causal prediction \citep[ICP;][]{peters2016causal}; and a non-causal baseline, from observational temporal $\chi^2$-correlations (CHI2). See App.~\ref{sec:causal_methods} for details of the embedding methods. The target task embeddings are inferred using expert feedback. We simulate an expert whose responses are generated according to the likelihood (Eq.~\eqref{eq:expert-likelihood}), with $\tau_\mathrm{expert}=10$. The simulated expert derives the task similarities using the same CI method as used in the construction of the source task embeddings, applied to a larger reference dataset that includes source tasks, target tasks, and also additional related disease endpoints. This reflects their broader domain knowledge. We also consider an oracle setting in which the meta-learner directly uses these embeddings from a reference dataset, bypassing the expert model, and providing an upper bound on performance. 

We compare our approach against the same meta-learning baselines in Section~\ref{sec:synthetic-expertiments}. The performance of our model is assessed with the three embeddings (CHI2, ICP, MR) considered, using both the expert-inferred embeddings and oracle embeddings. App.~\ref{app:ukbb} includes additional details for all methods. Table~\ref{tab:ukbb_results_auroc} displays the average AUROC (10 seeds) of each method in four out-of-distribution tasks. Across tasks, standard meta-learning baselines (MAML, DKT, HBM) generally improve over the ``no-transfer'' baseline indicating that task transfer is beneficial, although with modest and task-dependent gains. Our causal embedding-conditioned method further improves performance. 

For the embeddings inferred via simulated expert feedback, we observe improvements on tasks J44 and J45 when using ICP embeddings. When the oracle embeddings are available, ICP yields consistent improvements across all target tasks. In contrast, MR and CHI2 exhibit less consistent behavior across tasks, reflecting differences in the task similarities induced by each method. As shown in App.~\ref{app:embedding-analysis}, ICP embeddings induce a more structured embedding space, and identify J44 and J45 as the most OOD tasks relative to the source-task mean, which explains the stronger performance gains observed for these tasks. Additional results are provided in App.~\ref{app:supplementary-ukbb}.
\begin{table}[!ht]
\centering
% \scriptsize %
\footnotesize
\caption{Average AUROC (SD) on target tasks from the UKBB, over 10 seeds. CHI2, ICP, and MR correspond todifferent embeddings in our method. Best in \textbf{bold}.}
\label{tab:ukbb_results_auroc}
\setlength{\tabcolsep}{1,5pt}
\begin{tabular}{lccccc}
\toprule
\textbf{Method} & \textbf{J44} & \textbf{J45} & \textbf{G45} & \textbf{I21}  \\
\midrule %%%% using (SD) takes less space
CHI2, exp. & 0.730 (.028) & 0.571 (.015) & 0.631 (.013) & 0.690 (.015) \\
ICP, exp.  & \textbf{0.820 (.007)} & 0.620 (.007) & 0.655 (.007) & 0.693 (.008) \\
MR, exp. & 0.805 (.009) & 0.618 (.009) & 0.656 (.007) & 0.677 (.009) \\ \midrule
CHI2, oracle & 0.815 (.010) & 0.597 (.007) & 0.671 (.007) & 0.705 (.006) \\
ICP, oracle & 0.816 (.006) & \textbf{0.622 (.004)} & \textbf{0.673 (.007)} & \textbf{0.714 (.007)} \\
MR, oracle   & 0.766 (.011) & 0.602 (.006) & 0.651 (.009) & 0.711 (.008) \\ \midrule
No transfer & 0.791 (.010) & 0.587 (.010) & 0.646 (.010) & 0.704 (.007) \\ \midrule 
MAML & 0.791 (.008) & 0.612 (.006) & 0.667 (.012) & 0.701 (.006) \\
DKT & 0.809  (.009) & 0.601 (.010) & 0.652 (.013) & 0.695 (.012) \\
HBM & 0.797  (.007) & 0.609 (.007) & 0.670 (.005) & 0.701 (.007) \\ 
\bottomrule
\end{tabular}
\end{table}
\section{Conclusion}\label{sec:conclusion}
We have introduced a novel Bayesian meta-learning method that successfully leverages causal embeddings and expert-knowledge. Our approach uses the causal embeddings to set task-specific priors, while combining their information via a hierarchical Bayesian prior. Additionally, we provide theoretical guarantees over the generalization risk and the negative transfer, when including imperfect experts and causal discovery methods. The method we have proposed performs better under distribution shift than alternative methods, in synthetic cases and in a large multi-disease prediction task on the UK Biobank.

Future work that stems from our contributions could look into uncertainty propagation from the expert inference and the causal embeddings to the meta-learning model. Another possible avenue to explore is to learn the causal embeddings jointly during the meta-training from the source data. Additionally, propagating the uncertainty of the expert and the causal discovery to the theoretical results would provide even more realistic guarantees over the generalization risk.

% Finally, another direction of research that opens up is to generalize the definition of the causal embeddings and the manner in which they affect the model. For example, through neural networks instead of linear equations.

\newpage

% \section*{Software and Data}

\section*{Acknowledgements}
This work was supported by the Research Council of Finland (Flagship programme: Finnish Center for Artificial Intelligence FCAI and decisions 359567 and 358958), European Research Council (ODD-ML, grant agreement 101201120) ELLIS Finland, EU Horizon 2020 (European Network of AI Excellence Centres ELISE, grant agreement 951847) and UKRI Turing AI World-Leading Researcher Fellowship (EP/W002973/1). We also acknowledge the computational resources provided by the Aalto Science-IT Project from Computer Science IT. This research has been conducted using the UK Biobank Resource under application number 77565.
% \section*{Impact Statement}


\bibliography{references}
\bibliographystyle{icml2026}

%%%%%%%%%%%%%%%%%%%%%%%%%%%%%%%%%%%%%%%%%%%%%%%%%%%%%%%%%%%%%%%%%%%%%%%%%%%%%%%
%%%%%%%%%%%%%%%%%%%%%%%%%%%%%%%%%%%%%%%%%%%%%%%%%%%%%%%%%%%%%%%%%%%%%%%%%%%%%%%
% APPENDIX
%%%%%%%%%%%%%%%%%%%%%%%%%%%%%%%%%%%%%%%%%%%%%%%%%%%%%%%%%%%%%%%%%%%%%%%%%%%%%%%
%%%%%%%%%%%%%%%%%%%%%%%%%%%%%%%%%%%%%%%%%%%%%%%%%%%%%%%%%%%%%%%%%%%%%%%%%%%%%%%
\newpage
\appendix
\onecolumn
\section{Mathematical Derivations}
In this section we provide the formal mathematical proofs that by moving in the causal space, through the embeddings~($z$) we are able to move to different tasks. This is important because it guarantees that the method is able to capture new behaviours by changing solely through the embedding space. Additionally, we provide a Lipschitz-continuity proof of the error. This implies that small changes in the causal space do not catastrophically affect the error. 

%\JL{There are \textbf{many} papers that establish Lipschitz continuity of vanilla neural networks as a function of the inputs. Perhaps these would be relevant here, at least as motivation to what we are doing now.}

\subsection{Proof of Proposition~\ref{prop:eps-continuity}}\label{app:eps-continuity-full}
% Prior-induced risk is Lipschitz in embedding space}
We split the proof of Proposition~\ref{prop:eps-continuity} in two parts. First we demonstrate that the prior-induced risk is Lipschitz, and then we provide the error decomposition proof in Appendix~\ref{app:error-decomposition}.

\subsubsection{The Prior-induced risk is Lipschitz in the embedding space}\label{app:lipschitz}
\begin{proof}
For each task $t$, let $p_t$ denote the data-generating distribution over $(x,y)$, which does not depend on the model parameters $\phi$. For predictor $\phi$, define the task risk 
\[
\mathcal{R}_t(\phi)=\mathbb{E}_{(x,y)\sim p_t} [\ell(\phi;x,y)], \qquad \bar{\mathcal R}_t(z)
= \mathbb{E}_{\phi \sim p(\phi\mid z)}[\mathcal R_t(\phi)]
\]
and consider the embedding-conditioned prior 
\[
p(\phi \mid z) = \mathcal{N}(\theta+Wz, \sigma^2I).
\]
% \begin{lemma}[Continuity of prior-induced risk in embedding space]\label{lemma:lipschitz}
%     Assume $\ell(\phi;x,y) \in [0,M]$. Then for any task $t$ and embeddings $z_1,z_2 \in \mathcal{Z}$,
% \begin{equation}
%     \big|\bar{\mathcal R}_t(z_1)- \bar{\mathcal R}_t(z_2) \big| \leq \frac{M||W||}{2\sigma}\|z_1-z_2\|.
% \end{equation}
% In particular, if the causal embeddings $z_1,z_2$ are $\varepsilon$-similar, then 
% \begin{equation}
%    \big|\bar{\mathcal R}_t(z_1)- \bar{\mathcal R}_t(z_2) \big| \leq \frac{M||W||}{2\sigma}\varepsilon.
% \end{equation}
%\end{lemma}
Let $\ell(\phi;x,y)\in[0,M]$ be a bounded loss, and for $i=1,2$, denote by $p_i=p(\phi\mid z_i)$. Then, if $g(\phi)= M^{-1}\mathcal{R}_t(\phi)$, we have that
\begin{align*}
    \left|\bar{\mathcal R}_t(z_1)- \bar{\mathcal R}_t(z_2) \right| &= M\left|\mathbb E_{p_1}[g]-\mathbb E_{p_2}[g]\right| \\
    & = M\,\left|\mathbb E_{p_1}[g]-\mathbb E_{p_2}[g]\right| \\
    & \le M\,D_{\mathrm{TV}}(p_1,p_2),
\end{align*}
where the last line follows by definition of total variation distance \citep[Section~2.4, pp.~83--84,][]{tsybakov2009introduction}.
    
Next, applying Pinsker’s inequality \citep[Section~2.4, pp.~88--89,][]{tsybakov2009introduction}, we obtain
\[
D_{\mathrm{TV}}(p_1,p_2)
\le
\sqrt{\tfrac{1}{2}D_{\mathrm{KL}}(p_1\|p_2)}.
\]
 
Since $p_1$ and $p_2$ share covariance $\sigma^2I$, their KL divergence admits the closed form
\[
D_{\mathrm{KL}}\big(p(\phi\mid z_1)\|p(\phi\mid z_2)\big)
=
\frac{1}{2\sigma^2}\|W(z_1-z_2)\|^2
\le
\frac{\|W\|^2}{2\sigma^2}\|z_1-z_2\|^2.
\]
Applying the bound above yields:
\[
\big|\bar{\mathcal R}_t(z_1)-\bar{\mathcal R}_t(z_2)\big|
\le
\frac{M\|W\|}{2\sigma}\|z_1-z_2\|.
\]
Finally, since $\|z_1-z_2\|\leq \varepsilon$, due to $\varepsilon$-similarity of $z_1,z_2$, we obtain the desired result:
\[
\big|\bar{\mathcal R}_t(z_1)-\bar{\mathcal R}_t(z_2)\big|
\le
\frac{M\|W\|}{2\sigma}\,\varepsilon.
\]
\end{proof}
\subsubsection{Error decomposition on a target task}\label{app:error-decomposition}

\begin{proof}
Let $\hat{\phi}_{t'}^X$ denote the predictor obtained on a target task $t'$ by method $X$ and let $\mathcal{R}_{t'}(\hat{\phi}_{t'}^X)$ denote its target risk. For any embedding $z$, the prior-induced risk is given by
\[
\bar{\mathcal R}_{t'}(z)
= \mathbb{E}_{\phi \sim p(\phi\mid z)}[\mathcal R_{t'}(\phi)]
\]

Let $\hat{z}_{t'}$ denote the estimated embedding used for target task $t'$ at meta-test time, and $\bar{z} = \frac{1}{|\mathcal{T}_\mathrm{source}|}\sum_{t\in\mathcal{T}_\mathrm{source}}z_t$ is the empirical mean of the source embeddings. The prior mismatch is
\[\mathcal{E}_{\mathrm{prior}}(t') := |\bar{\mathcal R}_{t'}(\hat{z}_{t'}) - \bar{\mathcal R}_{t'}(\bar{z})|.\]

Denote by $z_{t'}$ the true latent causal embedding, $\tilde z_{t'}$ its causal discovery estimate, and $\hat z_{t'}$ the embedding used by the method, and define 
\[
\varepsilon_{\mathrm{expert}} := \|\hat z_{t'}-\tilde z_{t'}\|,
\quad
\varepsilon_{\mathrm{causal}} := \|\tilde z_{t'}-z_{t'}\|,
\quad
\varepsilon_{\mathrm{OOD}} := \|z_{t'}-\bar z\|.
\]

% \begin{lemma}[Prior mismatch bound]
% \label{lemma:prior-mismatch}

% Under the assumptions of Lemma~\ref{lemma:lipschitz},
% the prior mismatch satisfies
% \[
% \mathcal{E}_{\mathrm{prior}}(t')
% \le
% \frac{M\|W\|}{2\sigma}
% \big(
% \varepsilon_{\mathrm{expert}}
% +
% \varepsilon_{\mathrm{causal}}
% +
% \varepsilon_{\mathrm{OOD}}
% \big).
% \]
% \end{lemma}
Using the result proved in Appendix~\ref{app:lipschitz}, we obtain
\begin{equation*}
    \mathcal{E}_{\text{prior}}(t')\leq \frac{M||W||}{2\sigma}\|\hat{z}_{t'}-\bar{z}\|.
\end{equation*}
Using the triangle inequality on $d_{\mathcal{Z}}$, with intermediate embeddings $z_{t'},\tilde{z}_{t'},\hat{z}_{t'}$, we get 
\begin{equation*}% \label{eq:triangle-ineq}
    \|\hat{z}_{t'}-\bar{z}\|  \leq \underbrace{\|\hat{z}_{t'}-\tilde{z}_{t'}\|}_{\varepsilon_{\text{expert}}} + \underbrace{\|\tilde{z}_{t'}-z_{t'}\|}_{\varepsilon_{\text{causal}}} + \underbrace{\|z_{t'}-\bar{z}\|}_{\varepsilon_{\text{OOD}}}.
\end{equation*}
Combining the inequalities above the desired result is obtained:
\begin{equation*}% \label{eq:app-main-bound}
    \mathcal{E}_{\text{prior}}(t')\leq \frac{M||W||}{2\sigma}\Big(\varepsilon_{\text{expert}}+ \varepsilon_{\text{causal}}+\varepsilon_{\text{OOD}}\Big).
\end{equation*}
\end{proof}

\subsection{Full Statement and Proof of Theorem~\ref{th:negative-transfer}}\label{app:negative-transfer}
\begin{remark}\label{rem:lln}
    If the source embeddings are sampled from a distribution with mean $0$, then $\bar{z}\rightarrow0$, when the number of tasks $\big|\mathcal{T}_\mathrm{source}\big|\rightarrow\infty$ by a standard application of the law of large numbers, making the $\varepsilon_{\mathrm{OOD}}\approx\|z_{t'}\|$.  % \JL{Wouldn't we want this to be close to zero instead? For instance, it could be close to the distance with the nearest $z_t$?}
\end{remark}
% \JL{moved this from S.6, reworking it into the proof}
% We compare a causal-embedding meta-learning method to a global prior baseline. Since both methods are evaluated relative to the same no-transfer predictor, the no-transfer cancels. As such, it reduces to comparing the population risks between the predictors
% \[
% \mathrm{NT}_{\mathrm{causal}}(t') - \mathrm{NT}_{\mathrm{glob}}(t')
% =
% \mathcal R_{t'}(\hat\phi_{t'}^{\mathrm{causal}})
% -
% \mathcal R_{t'}(\hat\phi_{t'}^{\mathrm{glob}}).
% \]
We start this section with a formal definition of negative transfer, following the definition of \citet{wang2019characterizing}. Next, we explicitly provide the two conditions that are stated as mild in the main text, and proceed with the proof of Theorem~\ref{th:negative-transfer}.
\begin{definition}\citep[Negative transfer,][]{wang2019characterizing}
Let $\hat{\phi}_{t'}^{\mathrm{NT}}$ denote a predictor trained on the target task $t'$ only (no transfer). For any transfer method $X$, the negative transfer is given by the difference in risks
\[\mathrm{NT}_X(t') := \mathcal{R}_{t'}(\hat{\phi}_{t'}^X)-\mathcal{R}_{t'}(\hat{\phi}_{t'}^{\mathrm{NT}}).\] 
\end{definition}
The assumptions needed to prove Theorem~\ref{th:negative-transfer} are: % \LM{Could we have in theorem "Under assumptions A8-10" or similar} \JL{Yes, I added: 'under mild assumptions' to the theorem statement. But now we need to explicitly state the theorem here.}
\begin{assumption}% [...]
\label{ass:risk-minimizer}
The population risk minimizer for task $t'$ satisfies
\[
\phi_{t'}^* := \arg\min_{\phi} \mathcal{R}_{t'}(\phi)= \theta + W z_{t'}.
\]    
\end{assumption}

%\begin{assumption}[Stable adaptation]
%\label{ass:stable-adaptation}
%Adaptation contracts towards the task optimum, such that 
%\[
%\|\hat{\phi}_{t'}^{X}-\phi_{t'}^*\|
%\leq
%\kappa \|\phi_{t'}^{X,0}-\phi_{t'}^*\|,
%\quad \kappa \in (0,1),
%\]
%where $\phi_{t'}^{X,0}$ denotes the initialization induced by prior $X$. 
%\end{assumption}
%This is a standard stability assumption for regularized adaptation and is satisfied by variational inference via the ELBO. \JL{indicate at least one reference where this is a standard assumption. Ideally it would be a review of sorts.}

\begin{assumption}\label{ass:adaptation-bounds}
    There exists constants $0<\kappa_0<\kappa$ such that for method $X$
    \[
    \kappa_0\|\phi_{t'}^{X,0}-\phi_{t'}^*\| \leq \|\hat{\phi}_{t'}^{X}-\phi_{t'}^*\| \leq \kappa\|\phi_{t'}^{X,0}-\phi_{t'}^*\|.
    \]
\end{assumption}

\begin{assumption}[Risk monotonicity]
\label{ass:risk-monotonicity}
The population risk $\mathcal{R}_{t'}(\phi)$ is increasing in
$\|\phi-\phi_{t'}^*\|$ in a neighborhood of $\phi_{t'}^*$. % \JL{Where is this assumption as well?} 
\end{assumption}

%\begin{innercustomthm}{5}
\textbf{Theorem 5.}\textit{
Under Assumptions \ref{ass:risk-minimizer}--\ref{ass:risk-monotonicity}, stable adaptation, a well-conditioned embedding map $W$, and if
\begin{equation*}
\varepsilon_{\mathrm{expert}} + \varepsilon_{\mathrm{causal}}
\;\leq\;
C\cdot\varepsilon_{\mathrm{OOD}},
\end{equation*}
for a constant $C$. Then, the causal prior mitigates negative transfer relative to the global prior, i.e.,
\begin{equation*}
\mathrm{NT}_{\mathrm{causal}}(t') \leq \mathrm{NT}_{\mathrm{glob}}(t').
\end{equation*}
}
% \end{innercustomthm}

\begin{proof}
By definition of negative transfer,
\[
\mathrm{NT}_{X}(t')
=
\mathcal{R}_{t'}(\hat{\phi}_{t'}^{X})
-
\mathcal{R}_{t'}(\hat{\phi}_{t'}^{\mathrm{NT}}).
\]
Comparing negative transfer under the causal and global priors with the predictors $\hat{\phi}_{t'}^{\text{causal}}$ and $\hat{\phi}_{t'}^{\text{glob}}$ yields
\begin{align*}
\Delta_{\mathrm{NT}}(t')
&:=
\mathrm{NT}_{\mathrm{causal}}(t')
-
\mathrm{NT}_{\mathrm{glob}}(t') \\
&=
\Big(\mathcal{R}_{t'}(\hat{\phi}_{t'}^{\text{causal}})
-
\mathcal{R}_{t'}(\hat{\phi}_{t'}^{\mathrm{NT}})\Big)
-
\Big(\mathcal{R}_{t'}(\hat{\phi}_{t'}^{\text{glob}})
-
\mathcal{R}_{t'}(\hat{\phi}_{t'}^{\mathrm{NT}})\Big) \\
&=
\mathcal{R}_{t'}(\hat{\phi}_{t'}^{\mathrm{causal}})
-
\mathcal{R}_{t'}(\hat{\phi}_{t'}^{\mathrm{glob}}),
\end{align*}
since the no-transfer baseline cancels.

Thus, showing that the causal prior mitigates negative transfer relative to the global prior, means we need to show that:
\begin{align}
    \mathcal{R}_{t'}(\hat{\phi}_{t'}^{\mathrm{causal}}) \leq  \mathcal{R}_{t'}(\hat{\phi}_{t'}^{\mathrm{glob}}).\label{eq:smaller-risk}
\end{align}
By Assumption~\ref{ass:risk-monotonicity}, this reduces to showing
\begin{align}
    & \| \hat{\phi}_{t'}^{\mathrm{causal}}-\phi_{t'}^*\| \leq \| \hat{\phi}_{t'}^{\mathrm{glob}}-\phi_{t'}^*\|.\nonumber
\end{align}
% where the second line follows by Assumption~\ref{ass:risk-monotonicity}.
Now, let $\phi_{t'}^{\mathrm{causal},0}=\theta+W\hat{z}_{t'}$ and $\phi_{t'}^{\mathrm{glob},0}=\theta$ denote the prior means. Under Assumption~\ref{ass:risk-minimizer}, the distances between the global prior and the causal prior to task optimum $\phi_{t'}^*$ are respectively
\begin{align}
\|\phi_{t'}^{\mathrm{glob},0}-\phi_{t'}^*\|
& = \|\theta-(\theta+Wz_{t'})\| =
\|W z_{t'}\|,\text{ and }\label{eq:e-ood}\\
\|\phi_{t'}^{\mathrm{causal},0}-\phi_{t'}^*\|
& = \|(\theta+W\hat{z}_{t'})-(\theta+Wz_{t'})\| =
\|W(\hat z_{t'}-z_{t'})\|\label{eq:er-causal}.
\end{align}
Note that by Remark~\ref{rem:lln}, the last equality of Equation~\eqref{eq:e-ood} corresponds to $\varepsilon_{\text{OOD}}$. Additionally, the equality of Equation~\eqref{eq:er-causal} is bounded by $\varepsilon_{\mathrm{expert}}+\varepsilon_{\mathrm{causal}}$, where the expert error is $\varepsilon_{\mathrm{expert}}=\|\hat{z}_{t'}-\tilde{z}_{t'}\|$ and the causal discovery error is ${\varepsilon_{\mathrm{causal}}=\|\tilde{z}_{t'}-z_{t'}\|}$.

Applying Assumption~\ref{ass:adaptation-bounds}, we have the bounds 
\[
\kappa_0\|\phi_{t'}^{\mathrm{glob},0}-\phi_{t'}^*\| \leq \|\hat{\phi}_{t'}^{\mathrm{glob}}-\phi_{t'}^*\|, \qquad \|\hat{\phi}_{t'}^{\mathrm{causal}}-\phi_{t'}^*\| \leq \kappa\|\phi_{t'}^{\mathrm{causal},0}-\phi_{t'}^*\|.
\]

This implies that $\|\hat{\phi}_{t'}^{\mathrm{causal}}-\phi_{t'}^*\| \leq \|\hat{\phi}_{t'}^{\mathrm{glob}}-\phi_{t'}^*\| $ whenever
\[
\kappa\|\phi_{t'}^{\mathrm{causal},0}-\phi_{t'}^*\| \leq \kappa_0\|\phi_{t'}^{\mathrm{glob},0}-\phi_{t'}^*\|
\]

Substituting Equations~\eqref{eq:e-ood} and~\eqref{eq:er-causal} we get
\[
\kappa\|W(\hat z_{t'}-z_{t'})\| \leq \kappa_0\|Wz_{t'}\|.
\]
Since $W$ is well-conditioned, this is simplified to 
\[
\|\hat z_{t'}-z_{t'}\| \leq \frac{\kappa_0}{\kappa}\|z_{t'}\|,
\]
a sufficient condition for the causal prior to be closer to $\phi_{t'}^*$ than the global prior.

Defining $\varepsilon_{\mathrm{OOD}} := \|z_{t'}\|$ using Remark~\ref{rem:lln}, this yields 
\[
\varepsilon_{\mathrm{expert}}+\varepsilon_{\mathrm{causal}}
\leq C\cdot
\varepsilon_{\mathrm{OOD}},
\]
with constant $C=\frac{\kappa_0}{\kappa}$. The desired result, Equation~\eqref{eq:smaller-risk}, follows.
% \LM{Edit the theorem to have "up to constant C"} --- Done! --JL
\end{proof}
\clearpage 
\section{Full Algorithms}\label{app:algorithms}

\subsection{Meta-Learning}
We describe the full meta-training and meta-testing algorithms used in our experiments. Our approach builds on hierarchical Bayesian meta-learning formulations, namely \citet{ravi_amortized_2018}, but is implemented independently and extended to incorporate task-dependent priors via causal task embeddings. The losses considered are:
\begin{align}
    \mathcal{L}_1(\psi_t, \mathcal{D}_t^{(s)}, z_t) 
    &:= - \mathbb{E}_{q_{\psi_t}(\phi_t \mid \mathcal{D}_t^{(s)})} 
    \Big[ \log p(\mathbf{y}_t^{(s)} \mid \mathbf{x}_t^{(s)}, \phi_t) \Big] + D_{\mathrm{KL}}\Big(q_{\psi_t}(\phi_t \mid \mathcal{D}_t^{(s)}) \;\|\; p(\phi_t \mid z_t, \theta)\Big), \label{eq:loss-1} \\
    \mathcal{L}_2(\psi_t, \mathcal{D}_t^{(q)}, \lambda) & := - \mathbb{E}_{q_{\psi_t}(\phi_t \mid \mathcal{D}_t^{(s)})} 
\Big[ \log p(\mathbf{y}_t^{(q)} \mid \mathbf{x}_t^{(q)}, \phi_t) \Big]  + D_{\mathrm{KL}}\Big(q_\lambda(\theta) \;\|\; p(\theta)\Big).\label{eq:loss-2}
\end{align}
During meta-training (Algorithm~\ref{alg:2level_meta_clean}), we sample a mini-batch of source tasks and split data of each task $\mathcal{D}_t$ into support $\mathcal{D}_t^{(s)}$ and query sets $\mathcal{D}_t^{(q)}$. We perform $K$ inner-loop stochastic variational updates on $\mathcal{D}_t^{(s)}$ by minimizing $\mathcal{L}_1$ (Equation~\eqref{eq:loss-1}) starting from embedding conditional prior to obtain a variational task posterior $\psi_t^{(k+1)}$ (Line~\ref{line:inner-loop}). The outer loop then updates both the embedding weights $W$ and the global variational posterior $\lambda$ (Lines~\ref{line:outer-loop1} and~\ref{line:outer-loop2}), by minimizing $\mathcal{L}_2$ (Equation~\eqref{eq:loss-2}) across tasks, including $\ell_2$ regularization on $W$. 
\renewcommand{\thealgorithm}{B.\arabic{algorithm}} %%% resetting algorithms to ensure they have the 'B' at the beginning and people know to look in appendices.
\setcounter{algorithm}{0}
\begin{algorithm}[!htb]
\caption{Causally-Aware Meta-Learning: meta-training}
\label{alg:2level_meta_clean}
\begin{algorithmic}[1]
\REQUIRE Task distribution $p(\mathcal{T})$; learning rates $\beta_\theta, \beta_\phi$; inner steps $K$; mini-batch sizes $M_\text{tasks}, M_\text{samples}$
\REQUIRE Task embeddings $\{z_t\}_{t \in \mathcal{T}}$, embedding prior weights $W$, regularization coefficient $\gamma_W$
\STATE \textbf{Initialize:} global posterior parameters $\lambda = (\mu_\lambda, \sigma_\lambda^2)$, embedding weights $W$
\WHILE{not done}
    \STATE Sample mini-batch of $M_\text{tasks}$ tasks $t \sim p(\mathcal{T})$
    \FORALL{tasks $t$ in mini-batch}
        \STATE Sample $M_\text{samples}$ from $\mathcal{D}_t$ and split into support $\mathcal{D}_t^{(s)}$ and query $\mathcal{D}_t^{(q)}$
        \STATE Compute: $\mu_t = \mu_\lambda + W z_t$
        \STATE Initialize task posterior $\psi_t^{(0)}$ from embedding-conditional prior $q_{(\mu_t, \sigma_t^2)}(\theta)$
        \FOR{$k = 0$ to $K-1$}
            \STATE $\psi_t^{(k+1)} \gets \psi_t^{(k)} - \beta_\phi \nabla_{\psi_t^{(k)}} \mathcal{L}_1(\psi_t^{(k)}, \mathcal{D}_t^{(s)}, \mu_t, \sigma_t^2)$\COMMENT{Inner loop adaptation} \label{line:inner-loop}
        \ENDFOR
    \ENDFOR
\STATE Update global posterior and embedding parameters \COMMENT{Outer-loop adaptation} 
\STATE $W \gets W - \beta_W \nabla_{W} \frac{1}{M_\text{tasks}} \sum_{t} \mathcal{L}_2(\psi_t^{(K)}, \mathcal{D}_t^{(q)}, \mu_t, \sigma_t^2) + \gamma_W \| W \|t_2^2$ \label{line:outer-loop1}
\STATE $\lambda \gets \lambda - \beta_\theta \nabla_\lambda \frac{1}{M_\text{tasks}} \sum_{t} \mathcal{L}_2(\psi_t^{(K)}, \mathcal{D}_t^{(q)}, \mu_t, \sigma_t^2)$\label{line:outer-loop2}

\ENDWHILE
\end{algorithmic}
\end{algorithm}

For the meta-test approach (Algorithm~\ref{alg:2level_adapt_entmax_clean}), we split the target task data $\mathcal{D}_{t'}$ into adaptation $\mathcal{D}_{t'}^{\mathrm{tr}}$ and test sets $\mathcal{D}_{t'}^{\mathrm{ts}}$. We initialize the prior for target task $t'$ from the global training posterior $\lambda$, conditioned on the target task embedding $z_{t'}$, and adapt in the inner loop (Line~\ref{line:inner-loop-B2}) on target adaptation data $\mathcal{D}_{t'}^{\mathrm{tr}}$ by minimizing $\mathcal{L}_1$ (Equation~\eqref{eq:loss-1}), making the final predictions on the held-out target data.  
\begin{algorithm}[H]
\caption{Causally-Aware Meta-Learning: meta-testing}
\label{alg:2level_adapt_entmax_clean}
\begin{algorithmic}[1]
\REQUIRE Target tasks $\mathcal{T}_\text{target}$, task embeddings $z_{t'}$, global training posterior $\lambda$, learned embedding weights $W_\mu$, inner steps $K$, learning rate $\beta_\phi$
\FORALL{target tasks $t' \in \mathcal{T}_\text{target}$}
    \STATE Split $\mathcal{D}_{t'}$ into adaptation set $\mathcal{D}_{t'}^\text{tr}$ and test set $\mathcal{D}_{t'}^\text{ts}$ 
    \STATE Compute: $\mu_{t'} = \mu_\lambda + W_\mu z_{t'}$
    \STATE Initialize task posterior $\psi_{t'}^{(0)} = (\mu_{t'}, \sigma_{t'}^2)$
    \FOR{$k = 0$ to $K-1$}        
        \STATE Sample mini-batch of $M_\text{samples}$ from $\mathcal{D}_{t'}^\text{tr}$
        \STATE $\psi_{t'}^{(k+1)} \gets \psi_{t'}^{(k)} - \beta_\phi \nabla_{\psi_{t'}^{(k)}} \mathcal{L}_1(\psi_{t'}^{(k)}, \mathcal{D}_{t'}^{\text{tr}}, \mu_{t'}, \sigma_{t'}^2)$ \label{line:inner-loop-B2} \COMMENT{Inner loop adaptation}
    \ENDFOR
    \STATE Use final $\psi_{t'}^{(K)}$ for predictions on $\mathcal{D}_{t'}^\text{ts}$
\ENDFOR
\end{algorithmic}
\end{algorithm}

\subsection{Expert Feedback Model}\label{app:expert-feedback-model}
In this section, we describe the expert feedback model (Algorithm~\ref{alg:expert-feedback}) introduced in Section~\ref{sec:expert-inference}. For each query $b=1,\dots,B$, we select a most informative pair of source tasks using Bayesian active learning by disagreement (BALD), where entropies $H[\cdot]$ are approximated via Monte Carlo samples from the current variational posterior $q_\varphi(z_{t'})$. We query the expert to determine which source task is closer to the target and add the observed response to the comparison set $\mathcal{C}$ before updating a variational posterior over the target task embedding $z_{t'}$ via stochastic variational inference. After $B$ queries, the posterior mean of the inferred embedding is returned. 
\begin{algorithm}[!h]
\caption{Expert-guided inference of target task embedding}
\label{alg:expert-feedback}
\begin{algorithmic}[1]
\REQUIRE Source task embeddings $\{z_i\}_{i \in \mathcal{T}_{\text{source}}}$, target task $t'$, query budget $B$
\STATE Initialize prior $p(z_{t'}) = \mathcal{N}(0, I_d)$
\STATE Initialize comparison set $\mathcal{C} \gets \emptyset$

\FOR{$b = 1, \dots, B$}
    \STATE Select query $\xi_b = (i_b, j_b)$ with BALD
    \[
    \xi_b = \arg\!\max_{(i,j)} H[y_{ij} \mid \mathcal{C}] - \mathbb{E}_{z_{t'} \sim q_\varphi}[H[y_{ij} \mid z_{t'}]].
    \]
    \STATE Query expert: is $i_b$ or $j_b$ closer to $t'$?
    \STATE Receive response $c_b \in \{0,1\}$.
    \STATE $\mathcal{C} \gets \mathcal{C} \cup \{(\xi_b, c_b)\}$.
    \STATE Update $q_\varphi(z_{t'})$ via SVI with likelihood $p(c_b \mid z_{t'}, \xi_b) = \Phi(\tau \cdot \Delta(\xi_b; z_{t'}))$.
\ENDFOR

\STATE \textbf{Return} $\hat{z}_{t'} = \mathbb{E}_{q}[z_{t'} \mid \mathcal{C}]$.
\end{algorithmic}
\end{algorithm}

\clearpage 
\section{Additional Experiments}

In this section we present additional experimental results. 

\subsection{Imperfect Causal Discovery}\label{app:additional-causal-noise}
We study the effect of imperfect causal discovery on meta-learning performance, by constructing noisy versions of the causal task embeddings. For each task task embedding $z_t$, we sample a task-specific perturbation direction $d_t \sim \mathcal{N}(0,I_4)$ and set $u_t = d_t/\|d_t\|$.
With a noise level $\sigma_c = \{0, 0.5, 0.8\}$, we define
\begin{equation*}
    \tilde z_t
    =
    \|z_t\| \cdot \frac{z_t + \sigma_c \|z_t\| u_t}{\|z_t + \sigma_c \|z_t\| u_t\|}.
\end{equation*}

This construction introduces directional misalignment in the embedding space while preserving the embedding scale. The resulting embeddings are used to condition the task-specific priors in the meta-learning while the data-generating process remains unchanged. The results are presented in Figure~\ref{fig:noise_ablation}, where AUROC as a function of $\varepsilon_\mathrm{OOD}$ is presented. We see that with moderate noise $\sigma_c=0.5$ the performance of meta-learning stays robust under task-shift and starts to degrade only under high noise $\sigma_c=0.8$.

\begin{figure}[h]
    \centering
    \includegraphics[width=0.5\linewidth]{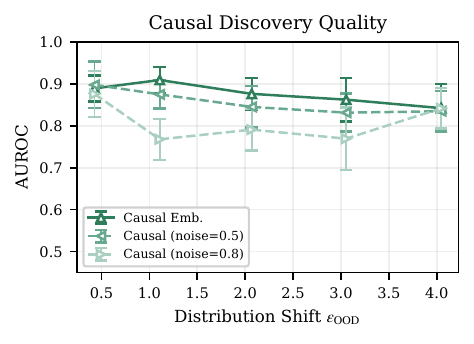}
    \caption{Performance of causally-aware meta-learning models under increasing noise in the causal task embeddings.}
    \label{fig:noise_ablation}
\end{figure}

\subsection{Expert Model Sensitivity to Expert Noise}\label{app:additional_expert_noise}

In this experiment, we assess the sensitivity of the expert model to expert unreliability. To simulate an unreliable expert, we vary the $\tau_\text{expert}$ parameter in the likelihood (Eq.~\eqref{eq:expert-likelihood}) that is used to  simulate the expert answers. Figure~\ref{fig:tau-sensitivity} shows the RMSE between inferred and true task embeddings as a function of the number of expert queries for target tasks ordered by increasing distribution shift for $\tau_\text{expert}=\{0.5, 1.0, 2.0\}$, with higher values corresponding to a more reliable expert. Across all tasks, lower expert noise leads to faster convergence and lower final error. For task 20 (magnitude of distribution shift $s=0.01$), the true task embedding lies close to the prior mean, explaining the almost zero RMSE before expert querying. 

\begin{figure}[H]
    \centering
    \includegraphics[width=1.0\linewidth]{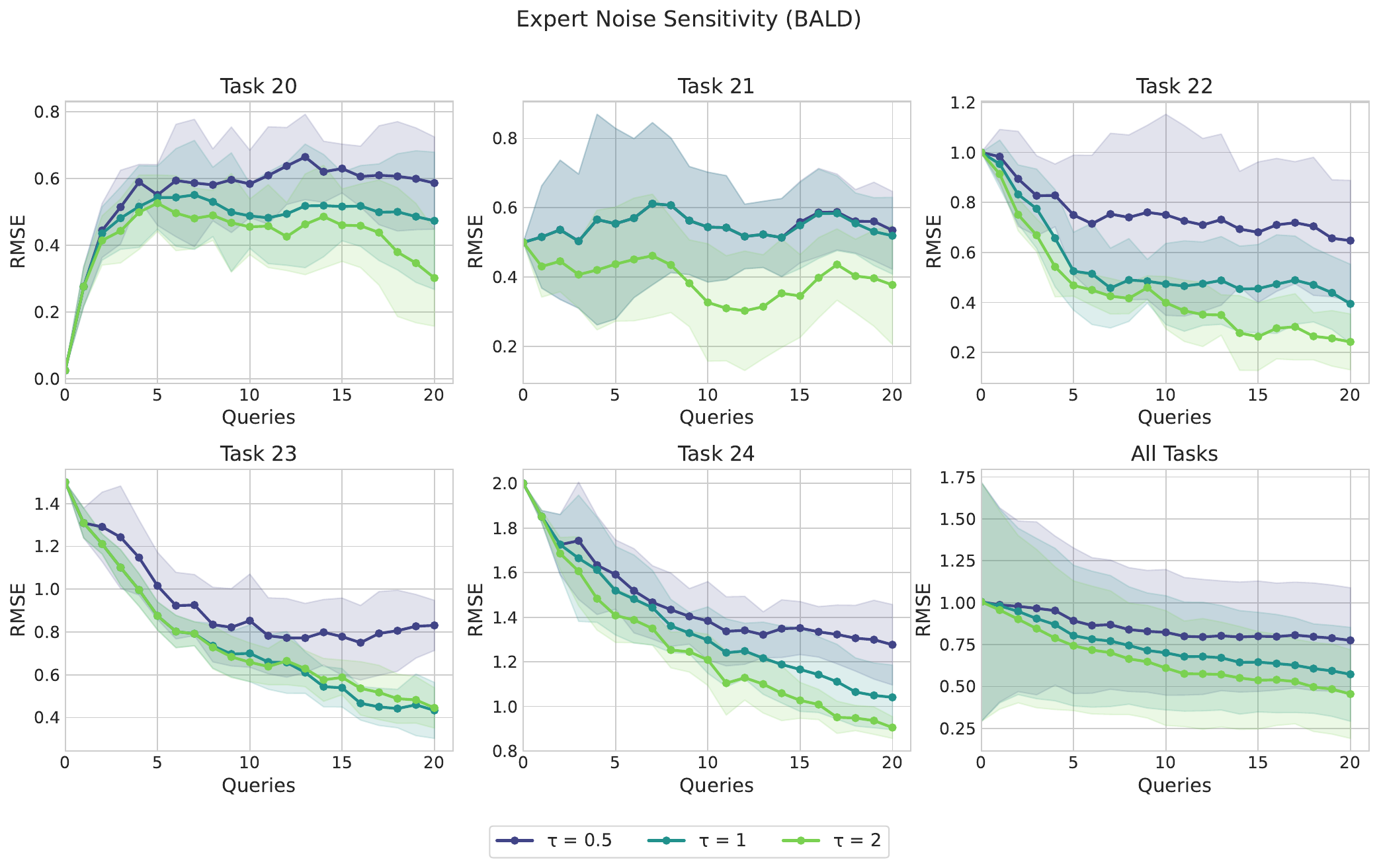}
    \caption{Expert-noise sensitivity under active querying. RMSE as a function of the number of expert queries at different noise levels $\tau_\text{expert}$ for multiple target tasks ordered with increasing distribution shift. Lower $\tau_\text{expert}$ corresponds to noisier expert feedback. The final panel aggregates results across all target tasks. }
    \label{fig:tau-sensitivity}
\end{figure}

\subsection{Active Query Strategy vs Random}\label{app:additional_bald_random}

In this experiment, we compare active query strategy with BALD against random querying, to assess if active querying leads to more accurate inference with smaller query budget. Figure~\ref{fig:bald-random} shows the RMSE between inferred and true task embeddings as a function of number of expert queries for target tasks ordered by increasing distribution shift for both query strategies at $\tau_\text{expert}=1.0$. We see that active querying achieves faster reductions in embedding error and lower final RMSE across tasks at high level of task shift. At low shift (tasks 20 and 21), the performance is similar active vs random, indicating that in in-distribution regimes, all queries are similarly informative. 

\begin{figure}[H]
    \centering
    \includegraphics[width=1.0\linewidth]{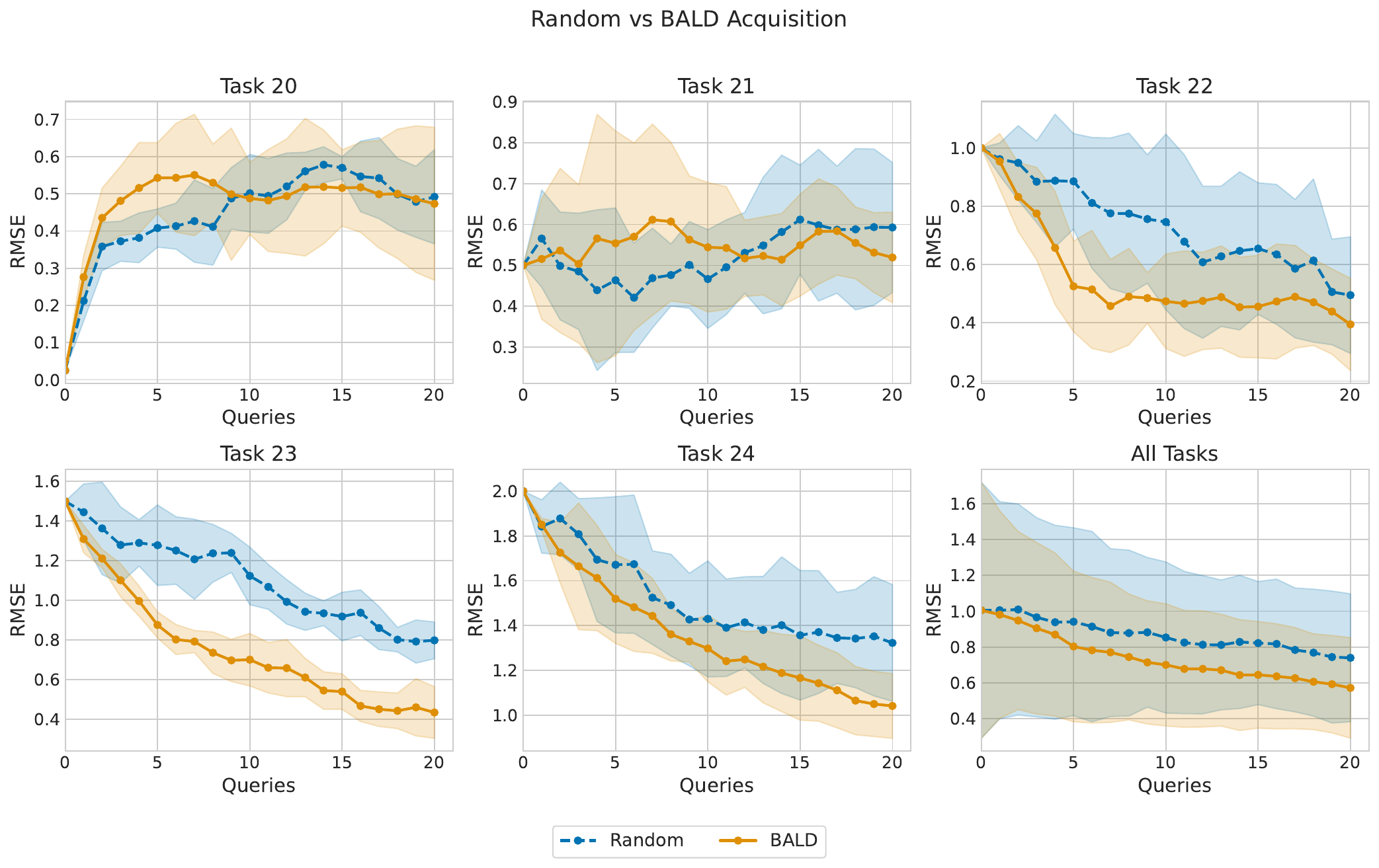}
    \caption{Random vs. active (BALD) querying. RMSE as a function of the number of expert queries for multiple target tasks ordered with increasing distribution shift, $\tau_\text{expert}=1.0$. The final panel aggregates results across all target tasks.}
    \label{fig:bald-random}
\end{figure}

\clearpage
\section{Experimental Details}\label{app:experimental-details}
In this section we describe the implementation details of our method and the other baselines, as well as the experimental settings for our synthetic experiment and the application to the cross-disease transfer in the UK Biobank. 

\subsection{Data Generation for Synthetic Experiments}\label{app:data-generation}
% \paragraph{Dataset Generation.} 
% Following the notation from Section~\ref{sec:background}, w
Following the notation from Section~\ref{sec:background-causal}, for each task $t \in \mathcal{T}=\mathcal{T}_{\mathrm{source}}\cup \mathcal{T}_{\mathrm{target}}$ we simulate a dataset $\mathcal{D}_t = \{(\mathbf{x}_{t,i},y_{t,i})\}_{i=1}^{500}$, by using a (task-specific) structural causal model given by $\mathcal{W}_t=(\mathbf{V}, \mathbf{U}, F_t, P_U)$. All tasks share the endogenous variables $\mathbf{V} = \{X_1, \ldots, X_{10}\}$ and the causal structure, but differ in their structural equations $F_t$. The values of $X_i$ are generated by sampling $\mathbf{x}_{t,i} \overset{i.i.d.}{\sim} \mathcal{N}(0,I_{10})$. For each task we assign a latent causal task embedding $z_t \in \mathcal{Z} \subseteq\mathbb{R}^4$, that parameterizes the task-specific data-generating mechanism. Source tasks ($|\mathcal{T}_{\text{source}}|=20$) are sampled from a centered distribution $z_t \overset{i.i.d.}{\sim} \mathcal{N}(0, (0.8)^2I_4)$, for $t \in \mathcal{T}_{\text{source}}$. Target tasks $t'$ ($|\mathcal{T}_{\text{target}}|=5$) are generated by moving along a fixed unit vector $\delta \in \mathbb{R}^4$ with $\|\delta\|=1$, $z_{t'} = s \cdot \delta$, for $t' \in \mathcal{T}_{\text{target}}$, where $s \in \{0.1, 1.0, 2.0, 3.0, 4.0\}$ controls the magnitude of distribution shift, with smaller values corresponding to smaller shifts. 

To generate binary outcome variables $Y_t$, we first generate the parent set $\mathcal{P} = \{1, 2,3,4\}$ (i.e., features $X_1,\ldots,X_4$), and intermediate continuous $Y^{(c)}_t$, by
\begin{equation*}
    Y_t^{(c)} = \sum_{j \in \mathcal{P}} X_j \, e_{t,j} + U_Y; \; U_Y\overset{i.i.d.}{\sim}\mathcal{N}(0, (0.6)^2).
\end{equation*}
The task-specific causal effects $e_t \in \mathbb{R}^{10}$ are generated by
\begin{equation*}
    e_t = b + Wz_t + \eta_t, \qquad \eta_t \overset{i.i.d.}{\sim} \mathcal{N}(0, (0.15)^2I).
\end{equation*}
The shared weight matrix $W\in\mathbb{R}^{10 \times 4}$ has rows $W_j \overset{i.i.d.}{\sim} \mathcal{N}(0, (0.5)^2 I_{4})$, for parent features $j \in \mathcal{P}$ and $W_j = 0$ otherwise. The bias vector $b \in \mathbb{R}^{10}$ has entries $b_j \overset{i.i.d.}{\sim} \mathrm{Unif}(0.5, 1.0)$, for $j \in \mathcal{P}$, and $b_j = 0$ otherwise. For non-parent features, $j \notin \mathcal{P}$, we set $e_{t,j}=0$.

To model realistic dataset bias and shortcut learning, we include the spurious feature $X_{10}$ with $10 \notin \mathcal{P}$: 
\begin{equation*}
    X_{10} = \alpha(s) \cdot (2Y_t - 1) + \epsilon, \qquad \epsilon \overset{i.i.d.}{\sim} \mathcal{N}(0, 1),
\end{equation*}
where $\alpha(s) = 0.5 \cdot (1 - s/s_{\max})$ with $s_{\max}=4.0$ for target tasks, and $\alpha=0.3$ for source tasks (resulting in a correlation of approximately $0.3$). $X_{10}$ is an observed feature present in every task, but it is not a causal parent of $Y_t$. Thus, it does not enter the structural equation for $Y_t^{(c)}$, and the spurious correlation decays with distribution shift magnitude $s$ while the causal mechanisms remains unchanged.

The final outcomes are obtained by thresholding at the $70^{\text{th}}$ percentile of the continuous distribution of $Y^{(c)}_t$ (per task). The values above the threshold give $Y_t=1$ and below give $Y_t=0$, resulting in approximately 30\% of positive labels per task.

\subsection{Expert Inference Model Implementation Details}
In this section, we explain the implementation details of our expert inference model. A standard Gaussian prior $p(z_{t'})=\mathcal{N}(0,I)$ is placed over the target task embedding, and we fix temperature $\tau=1.0$. For experiments, simulated expert responses may be generated with a separate noise parameter $\tau_\mathrm{expert}$ to assess robustness to model mismatch. We use query budget $B=20$ unless otherwise stated. The expert inference model is implemented in Pyro \citep{bingham2018pyro} and uses stochastic variational inference with diagonal Gaussian variational family (AutoDiagonalNormal) and the Trace ELBO objective with Adam optimizer. We use a learning rate of $0.01$, perform $150$ SVI steps after each expert query, and use $200$ Monte Carlo samples to approximate the BALD acquisition function. All experiments use fixed random seeds. 

\subsection{Meta-Learning and Baseline Implementation Details}

In this section we detail the implementation details of our method and baseline models. The complete implementation can be found in Supplementary material. 

\paragraph{Shared neural network architecture.}
The underlying predictive model for all methods is a Bayesian neural network (BNN) based on the long short-term memory \citep[LSTM;][]{hochreiter1997long} architecture. The architecture consists of two parallel networks with an LSTM layer for longitudinal data, and a fully connected multilayer perceptron (MLP) for tabular data. The outputs of the two networks are concatenated and passed through a linear layer, which produces logits to use in binary classification. Unless otherwise stated, all models use two hidden layers, with layer size of 32 for both MLP and LSTM components.

\paragraph{Hierarchical meta-Learning models.} 
Both meta-learning models are implemented in PyTorch. Task-specific and global model parameters are updated using diagonal variational inference operators implemented with the TorchOpt \citep{JMLR:TorchOpt} and Posteriors \citep{duffield2024scalable} libraries. The parameters of the embedding weight matrix $W$ are optimized separately using the Adam optimizer. 

To mitigate known optimization instabilities in gradient-based meta-learning, we apply several standard stabilization strategies. The embedding-conditioned prior adaptations are norm-constrained relative to the corresponding global parameter norms via an adaptation scale $\alpha$, and only the prior mean is adapted while the variance is kept fixed. Embedding conditioned weights in $W$ are initialized to zero and regularized using $\ell_2$ penalty with regularization coefficient $\gamma_W$, with additional gradient clipping ($\|\nabla_W\|_2 \leq 1.0$) and weight-norm constraints ($\|W\|_2 \leq 0.5$). During the variational inference, KL-divergence terms are scaled to account for mini-batching, and we employ tempered posteriors with temperature hyperparameters $T_1,T_2$. 

\paragraph{Bayesian neural network baseline.} For a no-transfer baseline, we train a separate BNN independently for each target task using the same predictive architecture described above. Each model is trained from scratch using variational inference, without any parameter sharing across tasks. Similar to the hierarchical meta-learning models, the BNN baseline is optimized using diagonal variational inference with TorchOpt and Posteriors libraries. 

\paragraph{Model agnostic meta-learning baseline.}
We are using gradient-based Model Agnostic Meta-Learning \citep[MAML;][]{finn2017model}, following the implementation in the original paper modified to work with our data loader and shared neural network architecture. 

\paragraph{Deep kernel transfer learning baseline.} We are using deep kernel transfer learning  \citep[DKT;][]{patacchiola2020bayesian}, following the implementation in the original paper modified to work with our data loader and shared neural network architecture. 

\subsection{Synthetic Experimental Settings}

The hyperparameters used for our synthetic experiments are detailed in Table~\ref{tab:hyperparam_final}. For all methods, task data are split into training (70\% of samples) and test sets (30\% of samples). The remaining data is further split into support (adaptation) and query (validation) datasets with 50-50\% split. We use two-fold cross validation for all methods. For MAML and DKT we use 100 samples per task. Bayesian predictive quantities are approximated using Monte Carlo sampling with $S=10$ posterior
samples. The convergence of the models is verified using early stopping based on validation AUROC. All reported results are averaged over multiple runs with fixed random seeds. 

\begin{table}[t]
\centering
\caption{Final hyperparameter values used in synthetic experiments.}
\label{tab:hyperparam_final}
\begin{tabular}{lll}
\toprule
\textbf{Method} & \textbf{Hyperparameter} & \textbf{Value} \\
\midrule
\multirow{4}{*}{BNN}
& Inner learning rate & $3\times10^{-3}$ \\
& Inner temperature & $10^{-1}$ \\
& Global prior std.\ $\sigma$ & $1.0$ \\
& Model init.\ log std. & $-1$ \\
\midrule
\multirow{5}{*}{MAML}
& Inner learning rate & $10^{-2}$ \\
& Outer learning rate & $10^{-3}$ \\
& Num. inner steps & $4$ \\
& Tasks per meta-batch & $4$ \\
& Samples per iteration & $100$ \\
\midrule
\multirow{4}{*}{DKT}
& Outer learning rate & $10^{-3}$ \\
& GP optimization steps & $50$ \\
& Tasks per meta-batch & $4$ \\
& Samples per iteration & $100$ \\
\midrule
\multirow{8}{*}{HBM}
& Inner learning rate & $10^{-4}$ \\
& Inner temperature & $5\times10^{-4}$ \\
& Global prior std.\ $\sigma$ & $0.05$ \\
& Prior scaling & $10^{4}$ \\
& Model init.\ log std. & $-3$ \\
& Num. inner updates & $4$ \\
& Outer learning rate & $10^{-3}$ \\
& Outer temperature & $5\times10^{-4}$ \\
\midrule
\multirow{11}{*}{Causal Meta-Learning}
& Inner learning rate & $10^{-4}$ \\
& Inner temperature & $5\times10^{-4}$ \\
& Global prior std.\ $\sigma$ & $0.05$ \\
& Prior scaling & $10^{4}$ \\
& Model init.\ log std. & $-3$ \\
& Num. inner updates & $4$ \\
& Outer learning rate & $10^{-3}$ \\
& Outer temperature & $5\times10^{-4}$ \\
& $W$ learning rate & $3\times10^{-4}$ \\
& $W$ regularizer $\gamma_W$ & $10^{-1}$ \\
& Adaptation scale $\alpha$ & $0.12$ \\
\bottomrule
\end{tabular}
\end{table}

\clearpage %%% to make clear distinction between subsections

\subsection{UKBB Dataset and Experimental Settings}\label{app:ukbb}
For the experiments in Section~\ref{sec:experiments} we use the UK Biobank (with project permission 77565). Next, we detail the data processing to construct the specific dataset used.

\paragraph{Dataset construction.}
The UK Biobank data includes 220,571 individuals with comprehensive primary care and hospital records. Baseline covariates comprised age, sex, ordinal lifestyle indicators (BMI category, smoking status, alcohol consumption), and continuous blood biochemistry measurements. Lifestyle indicators and continuous blood biochemistry variables were obtained at baseline, corresponding to the year of participant recruitment into UK Biobank. Individuals with missing lifestyle responses were excluded from further analyses, and missing biochemistry values were imputed using the cohort median. Summary statistics of the dataset are presented in Table~\ref{tab:cohort_summary}.

In addition, we constructed a longitudinal data set for the same patient cohort comprising of 304 diagnosis and medication variables, including ICD-10 diagnosis codes and medication codes encoded using the British National Formulary (BNF) and Read coding systems. These variables were aggregated at yearly intervals over the years 1990-2017. Each task was formulated as a binary classification problem, aiming to predict the occurrence of future outcomes after a fixed index year (2011) based on covariates observed prior to the index year. 

\begin{table}[h!]
\centering
\caption{Baseline characteristics of the UK Biobank cohort. Panel A reports demographic and lifestyle characteristics, and Panel B reports blood biochemistry measurements, along with their respective units of measurement. Continuous variables are reported as mean $\pm$ SD. Categorical variables are reported in counts (percentages).}
\label{tab:cohort_summary}
\begin{minipage}{0.48\linewidth}
\centering
\textbf{A: Demographic and lifestyle characteristics}
\begin{tabular}{l r}
\toprule
Participants, $N$ & 220{,}571 \\
Age (years) & 59.4 $\pm$ 8.1 \\
\midrule
\textbf{Sex} & \\
\quad Female & 121{,}370 (55.0\%) \\
\quad Male & 99{,}201 (45.0\%) \\
\midrule
\textbf{BMI category} & \\
\quad Underweight & 1{,}059 (0.5\%) \\
\quad Normal weight & 67{,}814 (30.7\%) \\
\quad Overweight & 94{,}862 (43.0\%) \\
\quad Obese & 56{,}836 (25.8\%) \\
\midrule
\textbf{Smoking status} & \\
\quad Never & 121{,}399 (55.0\%) \\
\quad Previous & 76{,}126 (34.5\%) \\
\quad Current & 23{,}046 (10.4\%) \\
\midrule
\textbf{Alcohol consumption} & \\
\quad Never & 17{,}856 (8.1\%) \\
\quad Low / occasional & 50{,}146 (22.7\%) \\
\quad Moderate / high & 152{,}569 (69.2\%) \\
\bottomrule
\end{tabular}
\end{minipage}
\hfill
\begin{minipage}{0.48\linewidth}
\centering
\textbf{B: Blood biochemistry measurements}
\begin{tabular}{lll}
\toprule
Alanine aminotransferase & (U/L) & 23.4 $\pm$ 13.8 \\
Alkaline phosphatase & (U/L) & 83.5 $\pm$ 25.1 \\
Apolipoprotein A & (g/L) & 1.53 $\pm$ 0.25 \\
Apolipoprotein B & (g/L) & 1.03 $\pm$ 0.23 \\
Aspartate aminotransferase & (U/L) & 26.1 $\pm$ 10.2 \\
C-reactive protein & (mg/L) & 2.52 $\pm$ 4.21 \\
Calcium & (mmol/L) & 2.38 $\pm$ 0.09 \\
Cholesterol & (mmol/L) & 5.70 $\pm$ 1.12 \\
Creatinine & ($\mu$mol/L) & 72.1 $\pm$ 17.9 \\
Cystatin C & (mg/L) & 0.91 $\pm$ 0.17 \\
Gamma glutamyltransferase & (U/L) & 36.7 $\pm$ 39.8 \\
Glucose & (mmol/L) & 5.10 $\pm$ 1.15 \\
HbA1c & (mmol/mol) & 36.1 $\pm$ 6.6 \\
HDL cholesterol & (mmol/L) & 1.44 $\pm$ 0.35 \\
IGF-1 & (nmol/L) & 21.4 $\pm$ 5.5 \\
LDL cholesterol & (mmol/L) & 3.56 $\pm$ 0.85 \\
Phosphate & (mmol/L) & 1.16 $\pm$ 0.15 \\
SHBG & (nmol/L) & 50.5 $\pm$ 25.6 \\
Total protein & (g/L) & 72.5 $\pm$ 3.8 \\
Triglycerides & (mmol/L) & 1.74 $\pm$ 1.00 \\
Urate & ($\mu$mol/L) & 308.2 $\pm$ 78.0 \\
Urea & (mmol/L) & 5.40 $\pm$ 1.35 \\
Vitamin D & (nmol/L) & 48.0 $\pm$ 20.0 \\
\bottomrule
\end{tabular}
\end{minipage}
\end{table}

\paragraph{Task selection.}
We select a group of diseases comprising of cardiovascular diseases and respiratory conditions, which are common in the UK Biobank and exhibit heterogeneous, but related, disease mechanisms. Tasks are split into source and target sets, where source tasks are used for meta-training and target tasks are held out for adaptation, as explained in Appendix~\ref{app:algorithms}. Key details of the selected source and target tasks are summarized in Table~\ref{tab:task_overview}. 
\begin{table}[H]
\centering
\caption{Overview of disease tasks used in the experiments, including ICD-10 codes, disease category (cardiovascular-metabolic = C or respiratory = R), clinical descriptions, task roles, and prevalence in the UK Biobank cohort (N=220,571).} 
\label{tab:task_overview}
\begin{tabular}{lcllrr}
\toprule
\textbf{ICD-10} & \textbf{Category} & \textbf{Disease description} & \textbf{Role} & \textbf{Cases} & \textbf{Prevalence (\%)} \\
\midrule
I21 & C & Acute myocardial infarction & Target & 3{,}941 & 1.79 \\
G45 & C & Transient ischaemic attack (TIA) & Target & 2{,}399 & 1.09 \\
J44 & R & Chronic obstructive pulmonary disease & Target & 5{,}850 & 2.65 \\
J45 & R & Asthma & Target & 4{,}318 & 1.96 \\
\midrule
E11 & C & Type 2 diabetes mellitus & Source & 9{,}571 & 4.34 \\
I10 & C & Essential (primary) hypertension & Source & 32{,}153 & 14.58 \\
I20 & C & Angina pectoris & Source & 5{,}325 & 2.41 \\
I25 & C & Chronic ischaemic heart disease & Source & 11{,}175 & 5.07 \\
I48 & C & Atrial fibrillation and flutter & Source & 9{,}704 & 4.40 \\
I63 & C & Cerebral infarction (ischaemic stroke) & Source & 2{,}898 & 1.31 \\
J18 & R & Pneumonia, organism unspecified & Source & 8{,}385 & 3.80 \\
J40 & R & Bronchitis, not specified acute/chronic & Source & 1{,}023 & 0.46 \\
J43 & R & Emphysema & Source & 1{,}474 & 0.67 \\
\bottomrule
\end{tabular}
\end{table}

\paragraph{Experimental settings.}
Bayesian optimization of the hyperparameters was done with Weights\&Biases, optimizing mean validation AUROC across target tasks. To ensure fair comparison, all methods were given comparable tuning budgets and identical data splits. Model architectures were fixed across methods (number of layers=2, number of hidden layers=32 in both tabular and longitudinal sets). All methods were tuned over 50 sweeps. The final experiments were done with the best hyperparameter configuration evaluated across 10 random seeds. The used ranges for hyperparameters are presented in Table~\ref{tab:hparam_ranges}. 
\begin{table}[H]
\centering
\caption{Hyperparameter tuning ranges. }
\label{tab:hparam_ranges}
\begin{tabular}{lll}
\toprule
\textbf{Method} & \textbf{Hyperparameter} & \textbf{Search Range} \\
\midrule
\multirow{5}{*}{BNN}
& Inner learning rate & log-uniform $[3\!\times\!10^{-5},\,3\!\times\!10^{-3}]$ \\
& Inner temperature & log-uniform $[10^{-4},\,3\!\times\!10^{1}]$ \\
& Global prior std.\ $\sigma$ & log-uniform $[0.02,\,0.2]$ \\
& Prior scaling & log-uniform $[10^{2},\,10^{4}]$ \\
& Model init.\ log std. & $\{-4,\,-3,\,-2\}$ \\
\midrule
\multirow{4}{*}{MAML}
& Inner learning rate & log-uniform $[10^{-4},\,5\!\times\!10^{-2}]$ \\
& Outer learning rate & log-uniform $[10^{-5},\,5\!\times\!10^{-3}]$ \\
& Num. inner steps & $\{1,\,2,\,4,\,8\}$ \\
& Tasks per meta-batch & $\{2,\,4,\,8\}$ \\
\midrule
\multirow{2}{*}{DKT}
& Outer learning rate & log-uniform $[10^{-5},\,10^{-2}]$ \\
& GP optimization steps & $\{25,\,50,\,100,\,200\}$ \\
\midrule
\multirow{7}{*}{Global Prior}
& Inner learning rate & log-uniform $[3\!\times\!10^{-5},\,3\!\times\!10^{-3}]$ \\
& Inner temperature & log-uniform $[10^{-4},\,10^{-2}]$ \\
& Global prior std.\ $\sigma$ & log-uniform $[0.02,\,0.2]$ \\
& Prior scaling & log-uniform $[10^{3},\,3\!\times\!10^{4}]$ \\
& Model init.\ log std. & $\{-4,\,-3,\,-2\}$ \\
& Num. inner updates & $\{2,\,4,\,6\}$ \\
& Outer learning rate & log-uniform $[3\!\times\!10^{-4},\,3\!\times\!10^{-3}]$ \\
\midrule
\multirow{10}{*}{Causal Prior}
& Inner learning rate & log-uniform $[3\!\times\!10^{-5},\,3\!\times\!10^{-3}]$ \\
& Inner temperature & log-uniform $[10^{-4},\,10^{-2}]$ \\
& Global prior std.\ $\sigma$ & log-uniform $[0.02,\,0.2]$ \\
& Prior scaling & log-uniform $[10^{3},\,3\!\times\!10^{4}]$ \\
& Model init.\ log std. & $\{-4,\,-3,\,-2\}$ \\
& Num. inner updates & $\{2,\,4,\,6\}$ \\
& Outer learning rate & log-uniform $[3\!\times\!10^{-4},\,3\!\times\!10^{-3}]$ \\
 & W learning rate &  log-uniform $[10^{-5},\,3\!\times\!10^{-3}]$\\ 
 & W regularizer $\gamma_W$ & log-uniform $[10^{-6},\,10^{-1}]$ \\
 & Adaptation scale $\alpha$&  \{0.1,\,0.2,\,0.3,\,0.5\}\\
\bottomrule
\end{tabular}
\end{table}

\subsection{Other Numerical Results on the UK Biobank Experiment}\label{app:supplementary-ukbb}

In this section, we report additional numerical results for all methods compared in the UKBB experiment, including complementary performance metrics and computational runtime statistics. Table~\ref{tab:ukbb_results_additional} shows the AUPRC and Matthews correlation coefficient (MCC) values across all methods. MCC is computed at a fixed 0.5 decision threshold for all methods. Across all tasks, causal methods consistently outperform no-transfer and standard meta-learning baselines. Notably, DKT consistently yields near-zero MCC, suggesting limited threshold level robustness under extreme class imbalance. 

Table~\ref{tab:ukbb_timing} reports average wall-clock time and peak memory usage on the UK Biobank experiments. The proposed causal meta-learning approach exhibits comparable computational cost to hierarchical Bayesian meta-learning (HBM). Among the methods using expert-inferred embeddings, the MR-based method shows longer runtimes and slightly higher memory usage, consistent with delayed early stopping, indicating slower convergence.  
\begin{table}[H]
\centering
\footnotesize
\caption{Average (SD) AUPRC and MCC on the UK Biobank results (over 10 seeds) on target tasks. CHI2, ICP, and MR correspond to our method with different embeddings. Best in \textbf{bold}.}
\label{tab:ukbb_results_additional}
\setlength{\tabcolsep}{3pt}
\begin{tabular}{lcccccccc}
\toprule
 & \multicolumn{4}{c}{\textbf{AUPRC}} & \multicolumn{4}{c}{\textbf{MCC}} \\
\cmidrule(lr){2-5} \cmidrule(lr){6-9}
\textbf{Method} 
& \textbf{J44} & \textbf{J45} & \textbf{G45} & \textbf{I21}
& \textbf{J44} & \textbf{J45} & \textbf{G45} & \textbf{I21} \\
\midrule
CHI2, exp     & 0.074 (.013) & 0.029 (.002) & 0.017 (.001) & 0.039 (.004) & 0.122 (.020) & 0.030 (.007) & 0.040 (.005) & 0.075 (.008) \\
ICP, exp.     & 0.119 (.009) & 0.036 (.001) & 0.019 (.001) & 0.041 (.002) & \textbf{0.176 (.005)} 
& \textbf{0.055 (.004)} & 0.047 (.004) & 0.076 (.004) \\
MR, exp.      & 0.107 (.006) & \textbf{0.038 (.002)} & 0.020 (.001) & 0.039 (.002) & 0.164 (.009) 
& \textbf{0.055 (.004)} 
& 0.046 (.003) 
& 0.070 (.005) \\
\midrule
CHI2, oracle &\textbf{ 0.126 (.009)} & 0.032 (.001) & 0.020 (.001) & 0.043 (.003) & 0.174 (.013) & 0.044 (.005) & \textbf{0.053 (.003)} & 0.082 (.004) \\
ICP, oracle  & 0.120 (.007) & 0.036 (.002) &\textbf{ 0.021 (.001)} & \textbf{0.045 (.003)} & 0.172 (.007) & 0.054 (.003) & 0.052 (.004) &\textbf{ 0.085 (.005)} \\
MR, oracle   & 0.085 (.005) & 0.034 (.002) & 0.019 (.001) & 0.044 (.002) & 0.141 (.008) & 0.046 (.003) & 0.048 (.004) & \textbf{0.085 (.004)} \\
\midrule
No transfer & 0.099 (.009) & 0.031 (.002) & 0.018 (.001) & 0.040 (.002) & 0.155 (.007) & 0.037 (.004) & 0.045 (.004) & 0.084 (.004) \\
\midrule
MAML  & 0.099 (.004) & 0.036 (.002) & 0.020 (.001) & 0.042 (.003) & 0.151 (.005) & 0.051 (.003) & 0.051 (.005) & 0.079 (.002) \\
DKT   & 0.117 (.009) & 0.034 (.002) & 0.019 (.001) & 0.038 (.003) & 0.032 (.013) & 0.000 (.000) & 0.004 (.003) & 0.008 (.008) \\
HBM   & 0.102 (.006) & 0.035 (.002) & \textbf{0.021 (.001)} & 0.043 (.003) & 0.151 (.006) & 0.048 (.003) & 0.052 (.003) & 0.080 (.003) \\
\bottomrule
\end{tabular}
\end{table}

\begin{table}[H]
\centering
\caption{Computational cost of different methods on the UK Biobank experiments.
Reported values correspond to average wall-clock time and peak memory usage per task.}
\label{tab:ukbb_timing}
\begin{tabular}{lcc}
\toprule
\textbf{Method} & \textbf{Elapsed Time} & \textbf{Peak Memory} \\
 & \textbf{(hh:mm:ss)} & \textbf{(GiB)} \\
\midrule
CHI2, exp & 00:13:09 & 7.5 \\
ICP, exp & 00:12:13 & 7.1 \\
MR, exp & 00:20:39 & 7.7 \\
\midrule
CHI2, oracle & 00:08:11 & 7.1 \\
ICP, oracle & 00:08:21 & 7.0 \\
MR, oracle & 00:09:47 & 7.1 \\
\midrule
No transfer & 00:02:06 & 6.5 \\
\midrule
MAML & 00:03:33 & 6.5 \\
DKT & 00:03:11 & 6.6 \\
HBM & 00:13:27 & 7.1 \\
\bottomrule
\end{tabular}
\end{table}
% \clearpage

\section{Details for the Construction of the Task Embeddings}\label{sec:causal_methods}
In this section, we describe the causal and correlational methods used to obtain the task embeddings and analyze their differences. These tasks embeddings are employed in the UK Biobank experiments. We note that the causal inference methods we use employ the longitudinal structure of the dataset. To construct task embeddings we use those disease endpoints $s$ that correspond to other tasks $t$. %%% removing 'are implemented using the scripts from \citet{wharrie_bayesian_2024}, which were originally', and 'and are directly applicable to the UK Biobank setting', to avoid issues with self-citing. This can be modified for the final version and the arxiv version --- JL

\subsection{Causal and Correlational Methods used for Task Embeddings}
\paragraph{Mendelian randomisation.} 
We estimate causal effects between disease endpoints using two-sample MR implemented using the \texttt{TwoSampleMR} R package \citep{mrsteiger, twosamplemr}. We use genome-wide association studies (GWAS) summary statistics from the UK Biobank analyses released by  \citet{nealelab_ukbb}. Genetic instruments for each endpoint are derived from GWAS summary statistics  by filtering variants (MAF $>$ 0.01, $p<10^{-5}$) and applying LD clumping (window 10,000kb, $r^2<0.01$) using an EUR reference panel and PLINK. For each task $t$, we run MR across exposure endpoints and use the inverse-variance weighted estimate $\beta^{\mathrm{IVW}}$ as the causal effect size. The resulting effect estimates $\beta_{s,t}^{\mathrm{IVW}}$ of the exposures $s=1,\ldots, K$, on task $t$ are collected into the embedding $z_t=[\beta_{1,t}^{\mathrm{IVW}},\ldots, \beta_{K,t}^{\mathrm{IVW}}]$.

\paragraph{Invariant causal prediction.}
For each task $t$, we construct a binary outcome indicating whether $t$ occurs in the outcome period and use counts of other tasks observed during exposure years as candidate predictors. We apply invariant causal prediction \citep[ICP; ][]{peters2016causal} using the \texttt{InvariantCausalPrediction} R package \citep{ICPrPackage} with a random split of the cohort into two environments and a significance level of $\alpha=0.1$, with boosting-based variable preselection (\texttt{maxNoVariables}=10, and  \texttt{maxNoVariablesSimult}=5). For each task $t$ (endpoint), ICP returns coefficient estimates for the predictors $s=1,\ldots,K$, whose association with the outcome remains invariant across environments. We form a task embedding $z_t=[\beta_{1,t},\ldots, \beta_{K,t}]$ by averaging coefficient estimates $\beta$ across accepted sets and setting non-selected coefficients to zero. We use the standard ICP formulation, which assumes no hidden confounders. 

\paragraph{Chi-square ($\chi^2$) temporal association baseline.}
As a correlational baseline, we construct task embeddings using chi-square ($\chi^2$) association statistics computed from longitudinal disease co-occurrence data. For each task $t$, we perform a chi-square test between the occurrence of $t$ in an outcome period after the index year and the occurrence of each endpoint $s=1,\ldots,K$ in the longitudinal data during the exposure period before the index year. The resulting $\chi^2$ statistics are assembled into a vector representation for each task $z_t=[\chi^2_{1,t},\ldots, \chi^2_{K,t}]$. We compute chi-square statistics using the scikit-learn package \citep{scikit-learn}.

\subsection{Comparison of the Embeddings}\label{app:embedding-analysis}
In this section, we compare the task embeddings produced by the different methods considered, and examine how they differ in their resulting task similarity structures. 

Figure \ref{fig:pca-embeddings} shows the task embeddings obtained from different causal inference methods, projected onto their first two principal components (PC1 and PC2) showing the differences in latent embedding space geometry across methods. ICP-derived embeddings exhibit moderately structured geometry with tasks spread along two clinically meaningful axes with cardiovascular-metabolic diseases clustering together and respiratory diseases together. For MR, most of the variance is explained by PC1 (70.9\%), making most of the tasks highly concentrated in the embedding space with a few outliers.  $\chi^2$-association embeddings have much larger scale (three orders of magnitude), and tasks are widely spread in the embedding space. However, there is a meaningful clustering according to the disease groups. Figure~\ref{fig:knn_graphs} shows the three-nearest neighbor graphs for the causal task embeddings. ICP and CHI2 exhibit similar connective structure with diseases within a same group connected together. MR shows clearly different structure that does not follow the same clinical grouping. 
\begin{figure}[!h]
    \centering
    \includegraphics[width=1.0\linewidth]{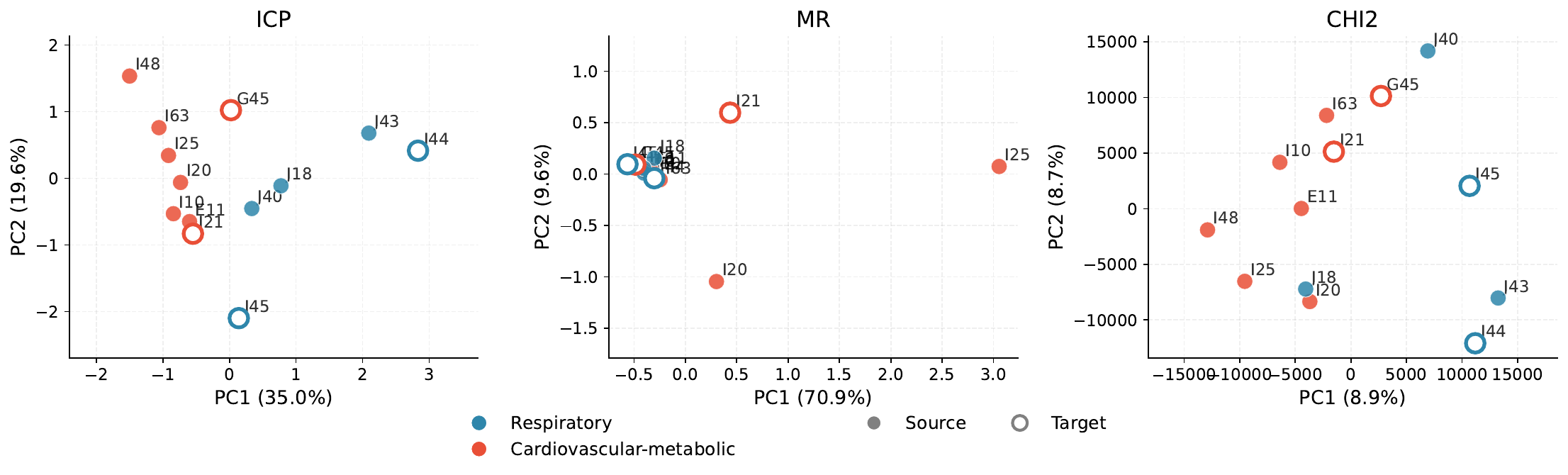}
    \caption{Principal component projections of causal task embeddings for Invariant Causal Prediction (ICP), Mendelian Randomisation (MR) and $\chi^2$ (CHI2). Colors indicate disease category (cardiovascular-metabolic vs. respiratory). Hollow markers indicate target tasks.}
    \label{fig:pca-embeddings}
\end{figure}

\begin{figure}[H]
    \centering
    \includegraphics[width=1.0\linewidth]{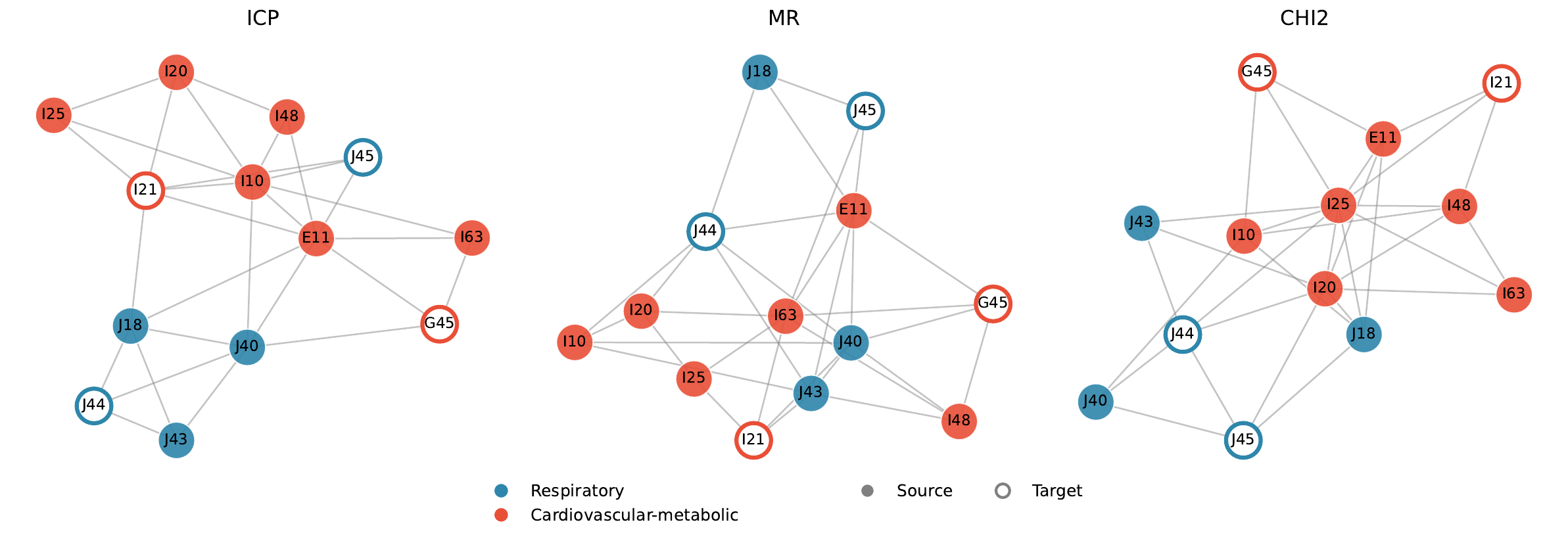}
    \caption{Three-nearest neighbor graphs of causal task embeddings. Nodes represent disease prediction tasks and edges connect tasks with similar causal structures. Colors indicate disease category (cardiovascular-metabolic vs. respiratory). Hollow markers indicate target tasks.}
    \label{fig:knn_graphs}
\end{figure}
The difference of the MR-derived embeddings, with respect to the other two methods, potentially reflects the heterogeneity in strengths of the genetic instruments across the diseases, which is a known challenge in MR \citep{burgess2011avoiding}. Diseases with well-powered GWAS (large sample size and strong genetic association) produce more stable causal estimates, while weaker instruments produce noisier estimated embeddings, which may not reflect the true causal structure.

Table~\ref{tab:ukbb_ood_distances} shows the OOD distances for each target task from to the average task ($\bar{z}$) and to the nearest source task. For the ICP method, J44 is the most OOD task with both metrics, and I21 is the least OOD in both metrics as well. MR ranks I21 (myocardial infarction, commonly known as heart attacks) as the most OOD task, which is clinically counterintuitive, since several tasks in the source are known risk factors, for instance type 2 diabetes (E11), and hypertension (I10). This reinforces that MR embeddings is capturing something that does not correspond to clinical similarity, and instead reflects the instrument strength. For CHI2, J45 is the most OOD task, but the distances are very similar throughout. 

\begin{table}[H]
\centering
\caption{Characterization of the target tasks in the embedding space, reported as the Euclidean distance and rank (in parenthesis, $1$ indicates the most OOD, $4$ is the least) to average source task ($\bar{z}$) and to the nearest source task. CHI2 distances are reported in units of $\times 10^3$. Most OOD task for each metric and method in \textbf{bold.}}
\label{tab:ukbb_ood_distances}
\setlength{\tabcolsep}{3pt}
\begin{tabular}{lcc cc cc}
\toprule
& \multicolumn{2}{c}{\textbf{ICP}} & \multicolumn{2}{c}{\textbf{MR}} & \multicolumn{2}{c}{\textbf{CHI2}} \\
\cmidrule(lr){2-3}\cmidrule(lr){4-5}\cmidrule(lr){6-7}
\textbf{Target}
& Avg. source & Nearest source
& Avg. source & Nearest source
& Avg. source & Nearest source \\
\midrule
J44 & \textbf{3.39 (1)} & \textbf{2.29 (1)} & 0.69 (4) & 0.59 (3) & 29.19 (4) & 38.63 (4) \\
J45 & 2.61 (2) & 2.10 (3) & 0.94 (2) & 0.72 (2) & \textbf{29.54 (1)} & \textbf{39.39 (1)} \\
G45 & 1.93 (3) & 2.15 (2) & 0.79 (3) & 0.50 (4) & 29.46 (2) & 39.24 (2) \\
I21 & 1.42 (4) & 1.38 (4) & \textbf{1.02 (1)} & \textbf{1.23 (1)} & 29.34 (3) & 39.12 (3) \\
\bottomrule
\end{tabular}
\end{table}

%%%%%%%%%%%%%%%%%%%%%%%%%%%%%%%%%%%%%%%%%%%%%%%%%%%%%%%%%%%%%%%%%%%%%%%%%%%%%%%
%%%%%%%%%%%%%%%%%%%%%%%%%%%%%%%%%%%%%%%%%%%%%%%%%%%%%%%%%%%%%%%%%%%%%%%%%%%%%%%

\end{document}